\documentclass[10pt,twocolumn,letterpaper]{article}

\usepackage{cvpr}
\usepackage{times}
\usepackage{epsfig}
\usepackage{graphicx}
\usepackage{amsmath,amssymb,amsthm}
\usepackage{algorithm,algorithmic}
\usepackage{graphicx}
\usepackage{psfrag,epsfig,subfigure,color, epstopdf}

\input{def.set}

\cvprfinalcopy % *** Uncomment this line for the final submission

\def\cvprPaperID{1161} % *** Enter the CVPR Paper ID here

\ifcvprfinal\fi
\begin{document}

% /* title and authors */
\title{Bilinear Random Projections for Locality-Sensitive Binary Codes}

\author{Saehoon Kim and Seungjin Choi\\
Department of Computer Science and Engineering \\
Pohang University of Science and Technology, Korea\\
{\tt\small \{kshkawa,seungjin\}@postech.ac.kr}
}

\maketitle
%\thispagestyle{empty}

% /* abstract */
\begin{abstract}
Locality-sensitive hashing (LSH) is a popular data-independent indexing method for approximate
similarity search, where random projections followed by quantization hash the points from the database
so as to ensure that the probability of collision is much higher for objects that are close to each other than for those that are far apart.
Most of high-dimensional visual descriptors for images exhibit a natural matrix structure.
When visual descriptors are represented by high-dimensional feature vectors and long binary codes are assigned, a random projection matrix requires expensive complexities in both space and time.
In this paper we analyze a bilinear random projection method where feature matrices are transformed to
binary codes by two smaller random projection matrices.
We base our theoretical analysis on extending Raginsky and Lazebnik's result where random Fourier features are composed with random binary quantizers to form locality sensitive binary codes.
To this end, we answer the following two questions: (1) whether a bilinear random projection also yields similarity-preserving binary codes; (2) whether a bilinear random projection yields performance gain or loss, compared to a large linear projection.
Regarding the first question, we present upper and lower bounds on the expected Hamming distance between binary codes produced by bilinear random projections.
In regards to the second question, we analyze the upper and lower bounds on covariance between two bits of binary codes, showing that the correlation between two bits is small.
Numerical experiments on MNIST and Flickr45K datasets confirm the validity of our method.
\end{abstract}

% /* main manuscript begins */
\section{Introduction}
\label{sec:introduction}

Nearest neighbor search, the goal of which is to find most relevant items to a query given a pre-defined distance metric,
is a core problem in various applications such as classification \cite{ShakhnarovichG2006book}, object matching \cite{LoweDG2004ijcv},
retrieval \cite{NisterD2006cvpr}, and so on.
A naive solution to nearest neighbor search is linear scan where all items in database are sorted
according to their similarity to the query, in order to find relevant items, requiring linear complexity.
In practical applications, however, linear scan is not scalable due to the size of examples in database.
Approximate nearest neighbor search, which trades accuracy for scalability, becomes more important than ever.
Earlier work \cite{FriedmanJH77acmtoms,AryaS98jacm} is a tree-based approach that
exploits spatial partitions of data space via various tree structures to speed up search.
While tree-based methods are successful for low-dimensional data, their performance is not
satisfactory for high-dimensional data and does not guarantee faster search compared to linear scan \cite{GionisA99vldb}.

\begin{figure}[t]
\begin{center}
\includegraphics[width=1.1\linewidth]
{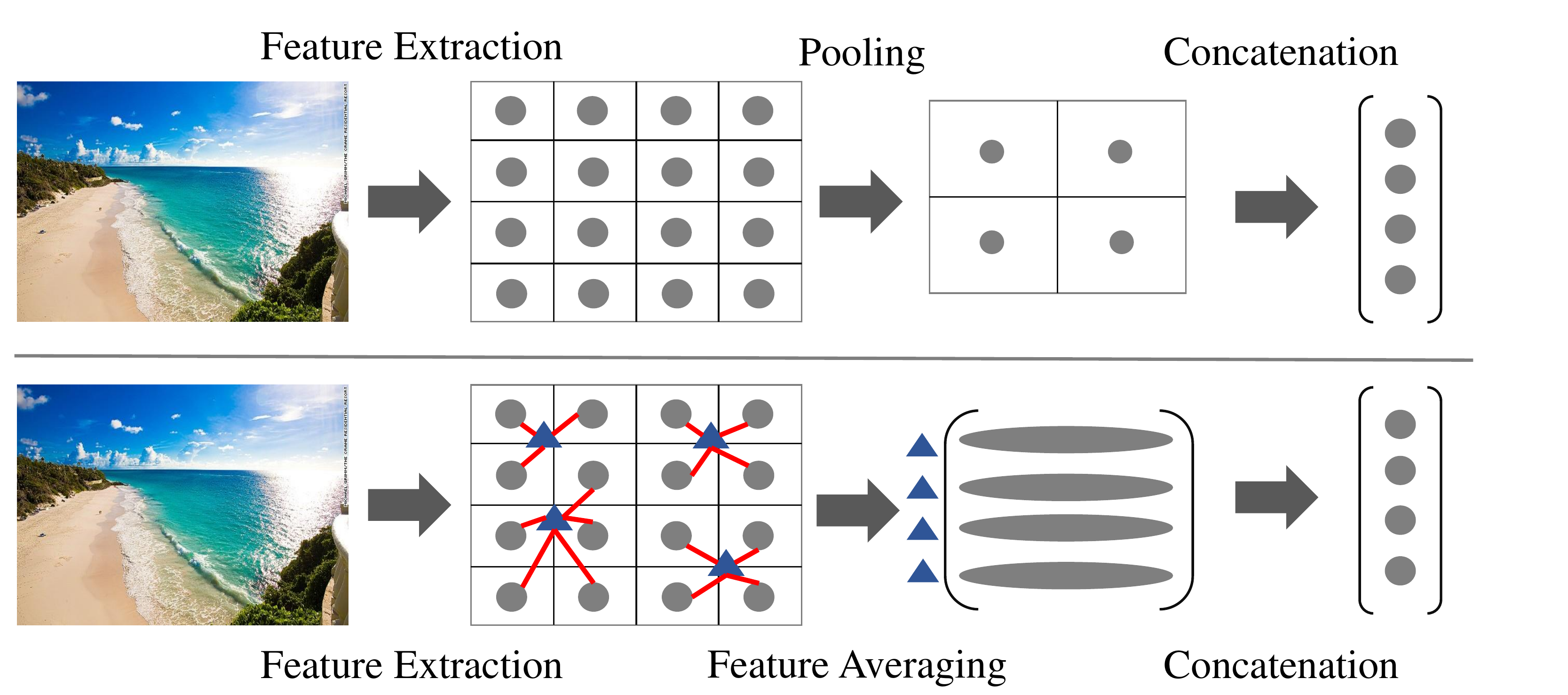}
\end{center}
\caption{Two exemplary visual descriptors, which have a natural matrix structure,
are often converted to long vectors.
The above illustration describes LLC  \cite{WangJ2010cvpr_b},
where spatial information of initial descriptors is summarized into
a final concatenated feature with spatial pyramid structure.
The bottom shows VLAD \cite{JegouH2010cvpr},
where the residual between initial descriptors and
their nearest visual vocabulary (marked as a triangle) is encoded in a matrix form.}
\label{fig:matrix_structure_feature}
\end{figure}

For high-dimensional data, a promising approach is approximate similarity search via hashing.
Locality-sensitive hashing (LSH) is a notable data-independent
hashing method, where randomly generates binary codes
such that two similar items in database are hashed to have high probability of collision \cite{GionisA99vldb,CharikarMS2002stoc,RaginskyM2009nips}.
Different similarity metric leads to various LSH, including angle preservation \cite{CharikarMS2002stoc}, $\ell_p$ norm ($p\in (0, 2]$) \cite{DatarM2004socg},
and shift-invariant kernels \cite{RaginskyM2009nips}.
Since LSH is a pure data-independent approach, it needs multiple hash tables or long code, requiring high memory footprint.
To remedy high memory consumption, data-dependent hashing \cite{SalakhutdinovR2007sigirworkshop,WeissY2007nips,WeissY2012eccv}
has been introduced to learn similarity-preserving binary codes from data such that embedding into a binary space preserves similarity between data points in the original space.
In general, data-dependent hashing generates compact binary codes, compared to LSH.
However, LSH still works well, compared to data-dependent hashing methods for a very large code size.

Most of existing hashing algorithms does not take into account
a natural matrix structure frequently observed in image descriptors
\cite{JegouH2010cvpr, WangJ2010cvpr_b}, as shown in Fig. \ref{fig:matrix_structure_feature}.
When a matrix descriptor is re-organized as a high-dimensional vector,
most of hashing methods suffer from high storage and time complexity due to a single large projection matrix.
Given a $d$-dimensional vector, the space and time complexities
to generate a code of size $k$ are both $O(dk)$.
In the case of $100,000$-dimensional data, 40GB\footnote{If we use single precision to represent a floating-point,
the projection matrix needs $100,000 \times 100,000 \times 4$ bytes
$\approx 40$GB.} is required to store a projection matrix
to generate a binary code of length $100,000$,
which is not desirable in constructing a large-scale vision system.

Bilinear projection, which consists of left and right projections, is a promising approach to
handling data with a matrix structure. It has been successfully applied to two-dimensional
principal component analysis (2D-PCA) \cite{YangJ2004ieeetpami} and
2D canonical correlation analysis (2D-CCA)\cite{LeeSH2007ieeespl}, demonstrating that the time and space
complexities are reduced while retaining performance, compared to a single large projection method.
Recently, bilinear projections are adopted to the angle-preserving LSH \cite{GongY2013cvpr},
where the space and time complexities are $O(\sqrt{dk})$ and $O(d\sqrt{k})$, to generate binary codes of size $k$
for a $\sqrt{d}$ by $\sqrt{d}$ matrix data.
Note that when such matrix data is re-organized as $d$-dimensional vector, the space and time complexities
for LSH are both $O(dk)$.
While promising results for hashing with bilinear projection are reported in \cite{GongY2013cvpr},
its theoretical analysis is not available yet.

In this paper we present a bilinear extension of LSH from shift-invariant kernels (LSH-SIK) \cite{RaginskyM2009nips}
and attempt to the following two questions on whether:
\begin{itemize}
\item randomized bilinear projections also yield similarity-preserving binary codes;
\item there is performance gain or loss when randomized bilinear projections are adopted
instead of a large single linear projection.
\end{itemize}
Our analysis shows that LSH-SIK with bilinear projections generates similarity-preserving
binary codes and that the performance is not much degraded compared to LSH-SIK with a large single linear projection.

\section{Related Work}
\label{sec:related}

In this section we briefly review LSH algorithms for preserving angle \cite{CharikarMS2002stoc}
or shift-invariant kernels \cite{RaginskyM2009nips}. We also review an existing bilinear hashing method \cite{GongY2013cvpr}.

\subsection{LSH: Angle Preservation}
\label{subsec:lsha}

Given a vector $\bx$ in $\Real^{d}$, a hash function $h_a(\bx)$ returns a
value $1$ or $0$, i.e., $h_a(\cdot): \Real^{d} \mapsto \{0,1\}$.
We assume that data vectors are centered, i.e., $\sum_{i=1}^{N} \bx_i = 0$, where $N$ is the number of samples.
The random hyperplane-based hash function involves a random projection
followed by a binary quantization, taking the form:
\be
\label{eq:ha}
h_a(\bx) \triangleq  \frac{1}{2} \Big\{ 1+\mbox{sgn} \left(\bw^\top\bx \right) \Big\},
\ee
where $\bw \in \Real^{d}$ is a random vector sampled on a unit $d$-sphere and
$\mbox{sgn}(\cdot)$ is the sign function which returns 1 whenever the input is nonnegative and 0 otherwise.
It was shown in \cite{CharikarMS2002stoc} that the random hyperplane method naturally gives
a family of hash functions for vectors in $\Real^{d}$ such that
\be
\label{eq:Ph}
\label{eq:LSH_collision_prob}
    \mathbb{P} \Big[ h_a(\bx) = h_a(\by) \Big] & = & 1 - \frac{\theta_{\bx, \by}}{\pi},
\ee
where $\theta_{\bx,\by}$ denotes the angle between two vectors $\bx$ and $\by$.
This technique, referred to as LSH-angle, works well for preserving an angle, but it does not
preserve other types of similarities defined by a kernel between two vectors.

\subsection{LSH: Shift-Invariant Kernels}
\label{subsec:lshsik}

Locality-sensitive hashing from shift-invariant kernels \cite{RaginskyM2009nips}, referred to as LSH-SIK,
is a random projection-based encoding scheme, such that the expected Hamming distance between the binary codes of
two vectors is related to the value of shift-invariant kernel between two vectors.
Random Fourier feature (RFF) is defined by
\be
\label{eq:rff}
\phi_{w}(\bx) \triangleq \sqrt{2} \cos(\bw^\top\bx + b),
\ee
where $\bw$ is drawn from a distribution corresponding to an underlying shift-invariant kernel,
i.e., $\bw \sim p_{\kappa}$, $b$ is drawn from a uniform distribution over $[0,2\pi]$,
i.e., $b \sim \mbox{Unif}[0,2\pi]$.
If the kernel $\kappa$ is properly scaled, Bochner's theorem guarantees
$\E_{w,b}  \big[ \phi_{w}(\bx) \phi_{w}(\by) \big] = \kappa(\bx-\by)$,
where $\E$ is the statistical expectation and $\kappa(\cdot)$ represents a shift-invariant kernel \cite{RahimiA2007nips}.

LSH-SIK builds a hash function, $h(\cdot): \Real^{d} \mapsto \{0,1\}$,
composing RFFs with a random binary quantization
\be
\label{eq:h}
    h(\bx) \triangleq \frac{1}{2} \left\{ 1+\mbox{sgn} \left( \frac{1}{\sqrt{2}}\phi_{w}(\bx) + t \right) \right\},
\ee
where $t \sim \mbox{Unif}[-1,1]$.
The most appealing property of LSH-SIK provides upper- and lower-bounds on the expected Hamming distance
between any two embedded points, which is summarized in Theorem \ref{thm:bounds_lsh}.

%:theorem 1
\begin{theorem}{\cite{RaginskyM2009nips}}
\label{thm:bounds_lsh}
Define the functions
\bee
    g_1(\zeta) & \triangleq & \frac{4}{\pi^2}(1-\zeta) \\
    g_2(\zeta) & \triangleq & \min \left\{ \frac{1}{2}\sqrt{1-\zeta},
    \frac{4}{\pi^2} \left(1-\frac{2}{3} \zeta \right) \right\},
\eee
where $\zeta \in [0,1]$, and $g_1(0)=g_2(0)=\frac{4}{\pi^2}$, $g_1(1)=g_2(1)=0$.
Mercer kernel $\kappa$ is shift-invariant, normalized, and satisfies $\kappa(\alpha \bx - \alpha \by) \leq \kappa(\bx-\by)$
for any $\alpha \geq 1$.
Then the expected Hamming distance between any two embedded points satisfies
\be
    g_1 (\kappa(\bx-\by)) \leq \E \Big[ \calI \left[h(\bx) \neq h(\by) \right] \Big] \leq g_2 (\kappa(\bx-\by)),
\ee
where $\calI[\cdot]$ is the indicator function which equals 1 if its argument is true and 0 otherwise.
\end{theorem}

The bounds in Theorem \ref{thm:bounds_lsh} indicate that binary codes determined by
LSH-SIK well preserve the similarity defined by the underlying shift-invariant kernel.

\subsection{Hashing with Bilinear Projections}

Most of high-dimensional descriptors for image, including HOG, Fisher Vector (FV), and VLAD,
exhibit a natural matrix structure.
Suppose that $\bX \in \Real^{d_w \times d_v}$ is a descriptor matrix.
The matrix is reorganized into a vector $\bx = \mbox{vec}(\bX) \in \Real^{d}$ where $d=d_w d_v$, and then
a binary code of size $k$ is determined by $k$ independent use of the hash function (\ref{eq:h}).
This scheme requires $O(dk)$ in space and time.

A bilinear projection-based method \cite{GongY2013cvpr} constructs a hash function
$H_a(\cdot): \Real^{d_w\times d_v} \mapsto \{0, 1\}^{k_wk_v}$ that is of the form
\be
\label{eq:Ha}
    H_a(\bX) \triangleq \frac{1}{2} \left\{1+\mbox{sgn} \Big(\mbox{vec} \left(\bW^\top\bX\bV \right) \Big) \right\},
\ee
where $\bW \in \Real^{d_w \times k_w}$ and $\bV \in \Real^{d_v \times k_v}$, to produce a binary code of size $k=k_w k_v$.
This scheme reduces space and time complexity to $O(d_wk_w + d_vk_v)$ and $O(d_w^2 k_w + d_v^2k_v)$, respectively,
while a single large linear projection requires $O(d_w d_v k_w k_v)$ in space and time.\footnote{
Recently, \cite{YuFX2014icml} proposes a circulant embedding, which is implemented by discrete Fourier transform,
to reduce the space and time complexities to $O(d)$ and $O(d\log d)$
when $(d=d_w \times d_v)$ and the code length is $d$. Even though the circulant embedding is faster than bilinear projections,
we believe that it is worth analyzing hashing with bilinear projections, because the implementation is simpler than \cite{YuFX2014icml}.}
Empirical results in \cite{GongY2013cvpr} indicate that a random bilinear projection produces
comparable performance compared to a single large projection.
However, its theoretical behavior is not fully investigated.
In the next section, we consider a bilinear extension of LSH-SIK (\ref{eq:h}) and present
our theoretical analysis.

\section{Analysis of Bilinear Random Projections}
\label{sec:main}

In this section we present the main contribution that is an theoretical analysis of a bilinear extension of LSH-SIK.
To this end, we consider a hash function $h(\cdot): \Real^{d_w \times d_v} \mapsto \{0,1\}$ that is of the form
\be
\label{eq:hbl}
    h(\bX) \triangleq \frac{1}{2} \Big\{1+\mbox{sgn} \Big(\cos \big(\bw^\top\bX\bv + b \big) + t \Big) \Big\},
\ee
where $\bw, \bv \sim \calN(0,\bI)$, $b \sim \mbox{Unif}[0,2\pi]$, and $t \sim \mbox{Unif}[-1,1]$.
With the abuse of notation, we use $h(\cdot)$ for the case of randomized bilinear hashing, however,
it can be distinguished from (\ref{eq:h}), depending on its input argument $\bx$ or $\bX$.
To produce binary code of size $k=k_w k_v$, the hash function $H(\cdot): \Real^{d_w \times d_v} \mapsto \{0,1\}^{k}$
takes the form:
\be
\label{eq:Hbl}
    H(\bX) \triangleq \frac{1}{2} \Big\{1+\mbox{sgn} \Big(\cos \big(\mbox{vec}(\bW^\top\bX\bV) + \bb \big) + \bt \Big) \Big\},
\ee
where each column of $\bW$ or of $\bV$ is independently drawn from spherical Gaussian with zero mean
and unit variance, each entry of $\bb \in \Real^{k}$ or of $\bt \in \Real^{k}$ is drawn uniformly from $[0,2\pi]$ and $[-1,1]$, respectively.

We attempt to answer two questions on whether:
(1) bilinear random projections also yield similarity-preserving binary codes like the original LSH-SIK;
(2) there is performance gain or degradation when bilinear random projections are adopted
instead of a large linear projection.

To answer the first question, we compute the upper and lower bound on the expected Hamming distance
$\E \Big[ \calI \big[h(\bX) \neq h(\bY) \big] \Big]$ between any two embedded points computed by bilinear LSH-SIK
with Gaussian kernel.
Compared to the the original upper and lower bounds for LSH-SIK \cite{RaginskyM2009nips}
with a single linear projection (Theorem \ref{thm:bounds_lsh}), our upper bound is the same
and lower bound is slightly worse when the underlying kernel is Gaussian.

Regarding the second question, note that some of bits of binary codes computed by
the hash function (\ref{eq:Hbl}) share either left or right projection (column vector of $\bW$ or $\bV$),
leading to correlations between two bits.
We show that the covariance between two bits is not high by analyzing
the upper and lower bounds on covariance between the two bits.

\subsection{Random Fourier Features}
\label{subsec:rffbl}

We begin with investigating the properties of random Fourier features \cite{RahimiA2007nips} in the case of bilinear projections,
since BLSH-SIK (an abbreviation of bilinear LSH-SIK) bases its theoretical analysis on these properties.
To this end, we consider bilinear RFF:
\be
\label{eq:rffbl}
\phi_{w,v}(\bX) \triangleq \sqrt{2} \cos(\bw^\top \bX \bv + b),
\ee
where $\bw, \bv \sim \calN(\0,\bI)$ and $b \sim \mbox{Unif}[0, 2\pi]$.

In the case of randomized linear map where $\bw \sim \calN(0,\bI)$,
$\E [\phi_{w}(\bx) \phi_{w}(\by)] = \kappa_g(\bx-\by)$, where $\kappa_g(\cdot)$ is Gaussian kernel.
Unfortunately, for the randomized bilinear map, $\E \big[ \phi_{w,v}(\bX) \phi_{w,v}(\bY) \big] \neq \kappa_g \big(\mbox{vec}(\bX-\bY) \big)$,
where Gaussian kernel defined as
\bee
\kappa_g \big(\mbox{vec}(\bX-\bY) \big)  \hspace*{-.2cm} & \triangleq & \hspace*{-.2cm}
\mbox{exp} \left\{ -\frac{1}{2} \| \mbox{vec}(\bX-\bY) \|_{2}^{2} \right\} \\
\hspace*{-.2cm}  & = & \hspace*{-.2cm}
\mbox{exp} \left\{ -\frac{1}{2}\mbox{tr} \Big[ (\bX-\bY)(\bX-\bY)^\top \Big] \right\},
\eee
where $\bX, \bY \in \Real^{d_w \times d_v}$ and the scaling parameter of Gaussian kernel is set as 1.
However, we show that $\E \big[ \phi_{w,v}(\bX) \phi_{w,v}(\bY) \big]$ is between $\kappa_g \big(\mbox{vec}(\bX-\bY) \big)$
and $\kappa_g \big(\mbox{vec}(\bX-\bY) \big)^{0.79}$, which is summarized in the following lemma.

%:lemma 1
\begin{lemma}
\label{lem:rffbl}
Define $\Delta = (\bX-\bY)(\bX-\bY)^\top$.
Denote by $\{\lambda_j\}$ leading eigenvalues of $\Delta$.
The inner product between RFFs is given by
\be
\label{eq:kb}
\E_{w,v,b} \big[\phi_{w, v}(\bX)\phi_{w, v}(\bY) \big] & = & \prod_j (1+\lambda_j)^{-\frac{1}{2}} \nonumber \\
& \triangleq & \kappa_{b}(\bX-\bY).
\ee
Then, $\kappa_{b}(\bX-\bY)$ is upper and lower bounded in terms of Gaussian kernel $\kappa_g(\mbox{vec}(\bX-\bY))$:
\small
\bee
    \qquad \kappa_g \big(\mbox{vec}(\bX-\bY) \big) \leq \kappa_{b} \big(\bX-\bY \big) \leq
    \kappa_g \big(\mbox{vec}(\bX-\bY) \big)^{0.79},
\eee
\normalsize
provided that the following assumptions are satisfied:
\begin{itemize}
\item $||\bX||_{F} \leq 0.8$, which can be easily satisfied by re-scaling the data.
\item $\lambda_1 \leq 0.28\sum_{i=2}^{d_w}\lambda_i$, which can be easily satisfied for large $d_w$.
\end{itemize}
\end{lemma}

\begin{proof}
\bee
\lefteqn{ \E_{w, v, b} \Big[\phi_{w, v}(\bX)\phi_{w, v}(\bY) \Big] } \\
    & = & \int \int \cos \Big(\bw^\top(\bX-\bY)\bv \Big)p(\bw)p(\bv) \, d\bw d\bv\\
    & = & \int  \kappa_g \Big((\bX-\bY)^\top\bw \Big) p(\bw) \, d\bw \\
    & = & (2\pi)^{-\frac{d_w}{2}} \int
    \exp \left\{ -\frac{\bw^\top(\bI + \Delta) \bw}{2} \right\} \, d\bw \\
    & = & \big| \bI + \Delta \big|^{-\frac{1}{2}},
\eee
where $|\cdot|$ denotes the determinant of a matrix.
The eigen-decomposition of $\Delta$ is given by $\Delta = \bU \bLambda \bU^\top$,
where $\bU$ and $\bLambda$ are eigenvector and eigenvalue matrices, respectively.
Then we have
\bee
\big| \bI + \Delta \big|^{-\frac{1}{2}} & = & \big| \bU(\bI+\bLambda)\bU^\top \big|^{-\frac{1}{2}} \\
    & = & \prod_{j}(1+\lambda_j)^{-\frac{1}{2}}.
\eee
Now we prove the following inequalities:
\small
\bee
    \qquad \kappa_g(\mbox{vec}(\bX-\bY)) \leq \kappa_{b}(\bX-\bY) \leq \kappa_g(\mbox{vec}(\bX-\bY))^{0.79}.
\eee
\normalsize

\noindent
{\em Lower bound:}
First, we can easily show the lower bound on $\kappa_{b}(\bX-\bY)$
with the following inequality: $ \kappa_g(\mbox{vec}(\bX-\bY)) = \prod_j \mbox{exp}(\lambda_j)^{-\frac{1}{2}}
\leq \prod_j (1+\lambda_j)^{-\frac{1}{2}} \triangleq \kappa_{b}(\bX-\bY)$,
because $1+\lambda_j \leq \mbox{exp}(\lambda_j)$.

\noindent
{\em Upper bound:}
Second, assuming that $||\bX||_{F} \leq 0.8$, we can derive the upper bound on $\kappa_{b}(\bX-\bY)$. Now, we can bound $\lambda_j$ with the following logic.
\small
\bee
    & & \mbox{tr}\left[(\bX-\bY)(\bX-\bY)^\top \right]
     \leq  (2*0.8)^2  = 2.56 \\
    &\Rightarrow & \sum_k \lambda_k \leq 2.56 \\
    & \Rightarrow & \lambda_1 \leq 0.56 \quad (\because \lambda_k \geq 0,
    \lambda_1 \leq  0.28\sum_{i=2}^{d_w} \lambda_i).
\eee
\normalsize
For $0\leq \lambda_j \leq 0.56$,
we know that $\mbox{exp}(\lambda_j)^{0.79} \leq 1+\lambda_j$,
leading to the upper bound on $\kappa_{b}(\bX-\bY)$, i.e., $\kappa_{b}(\bX-\bY) \leq \kappa_g(\mbox{vec}(\bX-\bY))^{0.79}$.
\end{proof}

\begin{figure*}[ht]
\begin{center}
\subfigure[4-by-4 matrix]{\includegraphics[scale=0.22]{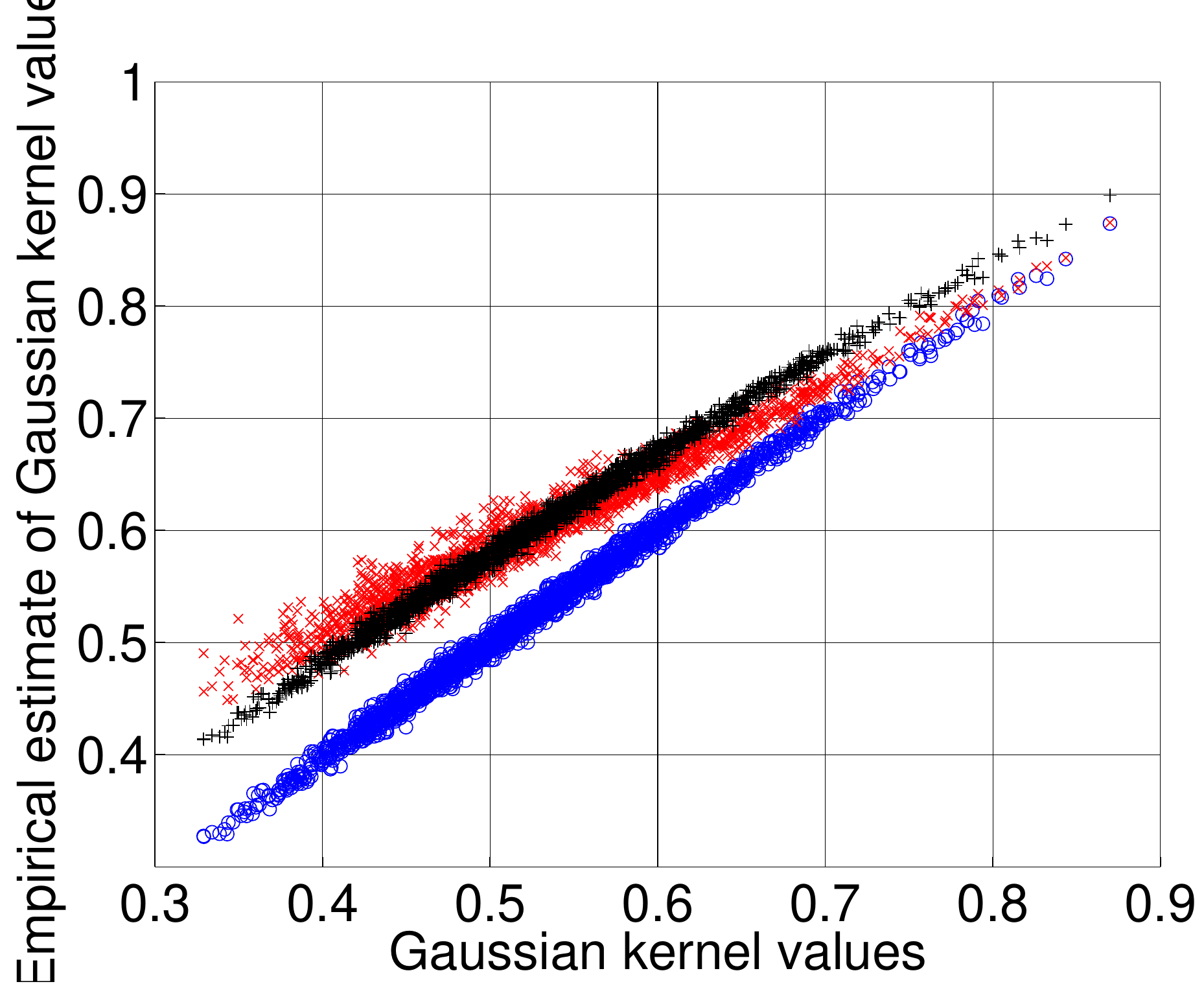}}
\subfigure[8-by-8 matrix]{\includegraphics[scale=0.22]{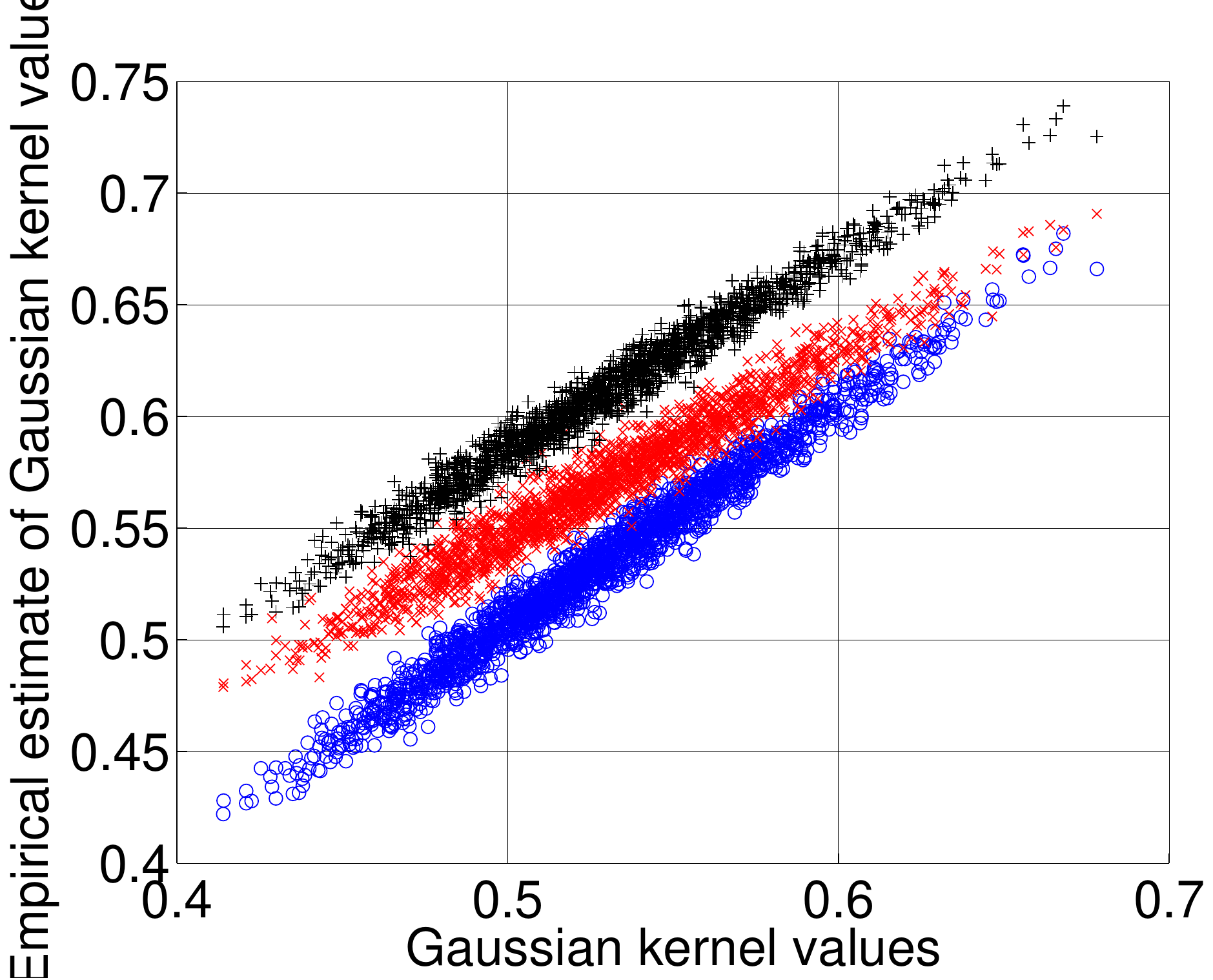}}
\subfigure[15-by-15 matrix]{\includegraphics[scale=0.22]{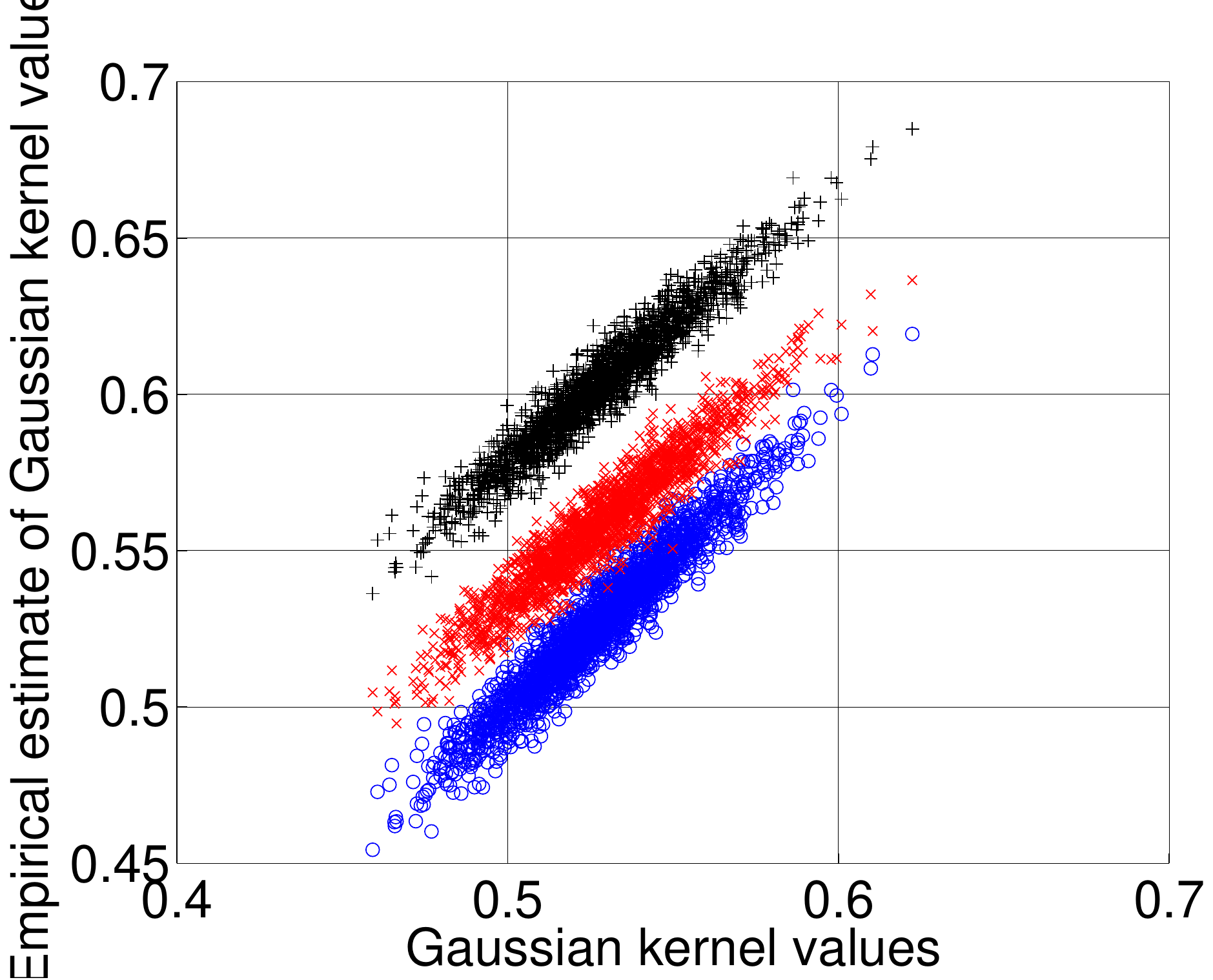}}
\subfigure[25-by-25 matrix]{\includegraphics[scale=0.22]{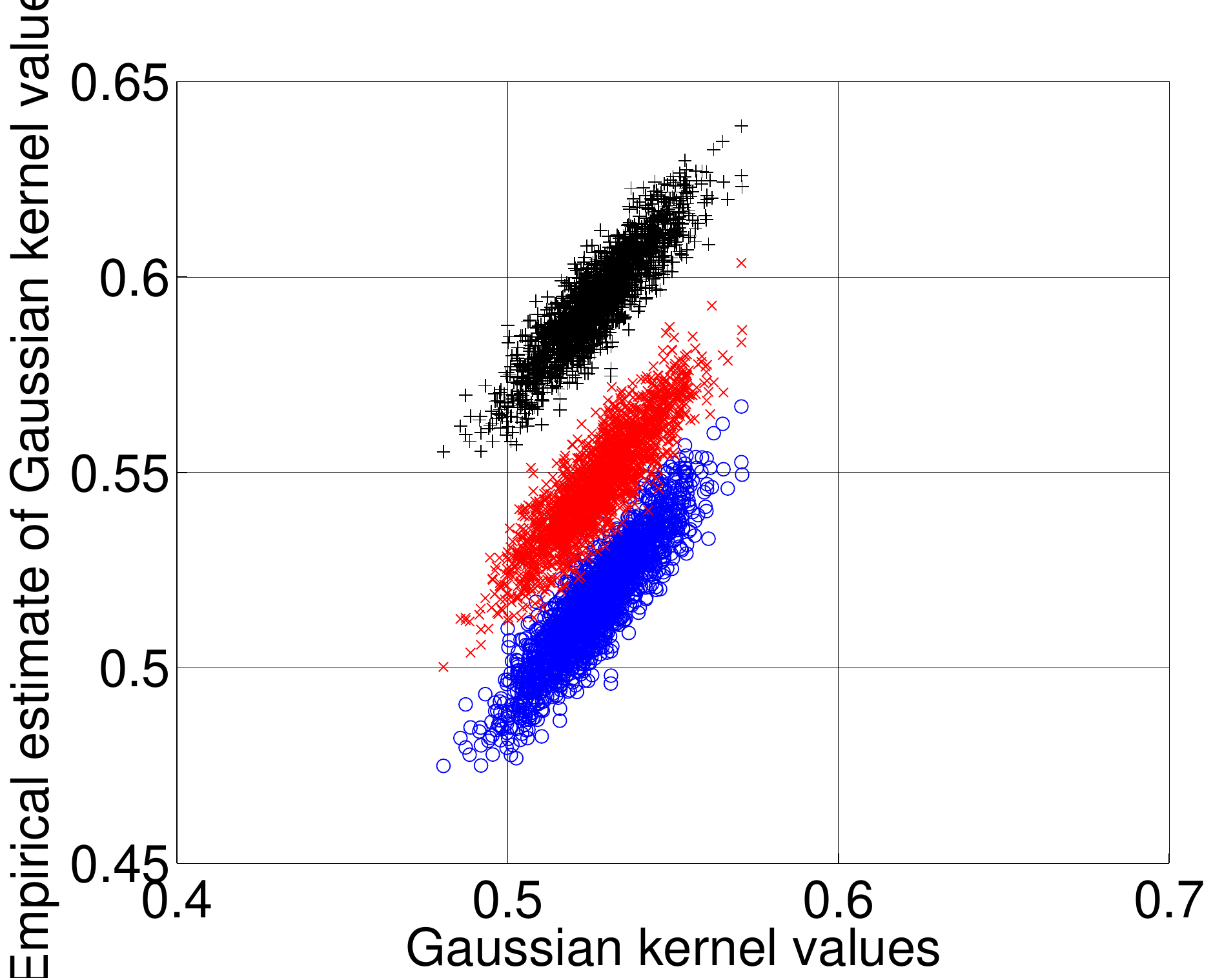}}
\end{center}
\caption{Estimates of bilinear RFF $(\kappa_b(\cdot))$ and its lower/upper bounds ($k_g(\cdot)$ and $k_g(\cdot)^{0.79}$) with respect to Gaussian kernel values. Red marks represent the inner products of two data points induced by bilinear RFF, and blue (black) marks represent its lower (upper) bounds.}
\label{fig:bilinear_RFF}
\end{figure*}

Lemma \ref{lem:rffbl} indicates that random Fourier features with bilinear
projections are related to the one with single projection in case of Gaussian kernel.
Due to this relation, we can conclude that random Fourier features with bilinear projections can generate similarity-preserving binary codes in the following section.
Finally, we summarize some important properties of $\kappa_{b}(\bX-\bY)$, showing that
$\kappa_{b}(\bX-\bY)$ shares the similar properties with $\kappa_g(\bX-\bY)$:

\begin{itemize}
\item {\bf Property 1}: $\qquad 0 \leq \kappa_{b}(\bX - \bY) \leq 1.$
\item {\bf Property 2}: $\qquad \kappa_{b}(m\bX - m\bY) \leq \kappa_{b}(\bX - \bY),$
where $m$ is a positive integer.
\end{itemize}

Fig. \ref{fig:bilinear_RFF} demonstrates that
the inner product of two data points
induced by bilinear RFF is upper and lower bounded with respect to Gaussian kernel
as shown in Lemma \ref{lem:rffbl}.
For the high $d_w$, the upper bound is satisfied in Fig. \ref{fig:bilinear_RFF} (c-d),
which is consistent with our intuition.

For Fig. \ref{fig:bilinear_RFF}, we generate the data from an uniform distribution
with different dimensions, and re-scale the data to be $||\bX||_{F}=0.8$.
To compute the estimates of bilinear RFF, we independently generate 10,000 triples $\{w_i, v_i, b_i\}$
and calculate the following sample average:
$\sum_{i=1}^{k}\big[\phi_{w_i,v_i}(\bX)\phi_{w_i, v_i}(\bY)\big]$, where $k=10,000$.
For the estimates of RFF, we calculate the sample average with 10,000
independently generated pairs $\{w_i, b_i \}$.

\subsection{Bounds on Expected Hamming Distance}
\label{subsec:bounds_ehd}

In this section, we derive the upper and lower bounds on the expected Hamming
distance between binary codes computed by BLSH-SIK to show that BLSH-SIK can generate similarity-preserving
binary codes in the sense of Gaussian kernel.
Lemma \ref{lem:ehd_blsh} is a slight modification
of the expected Hamming distance by LSH-SIK with a single projection \cite{RaginskyM2009nips},
indicating that the expected Hamming distance is analytically represented.

\begin{figure}[t]
\begin{center}
\includegraphics[scale=0.28]{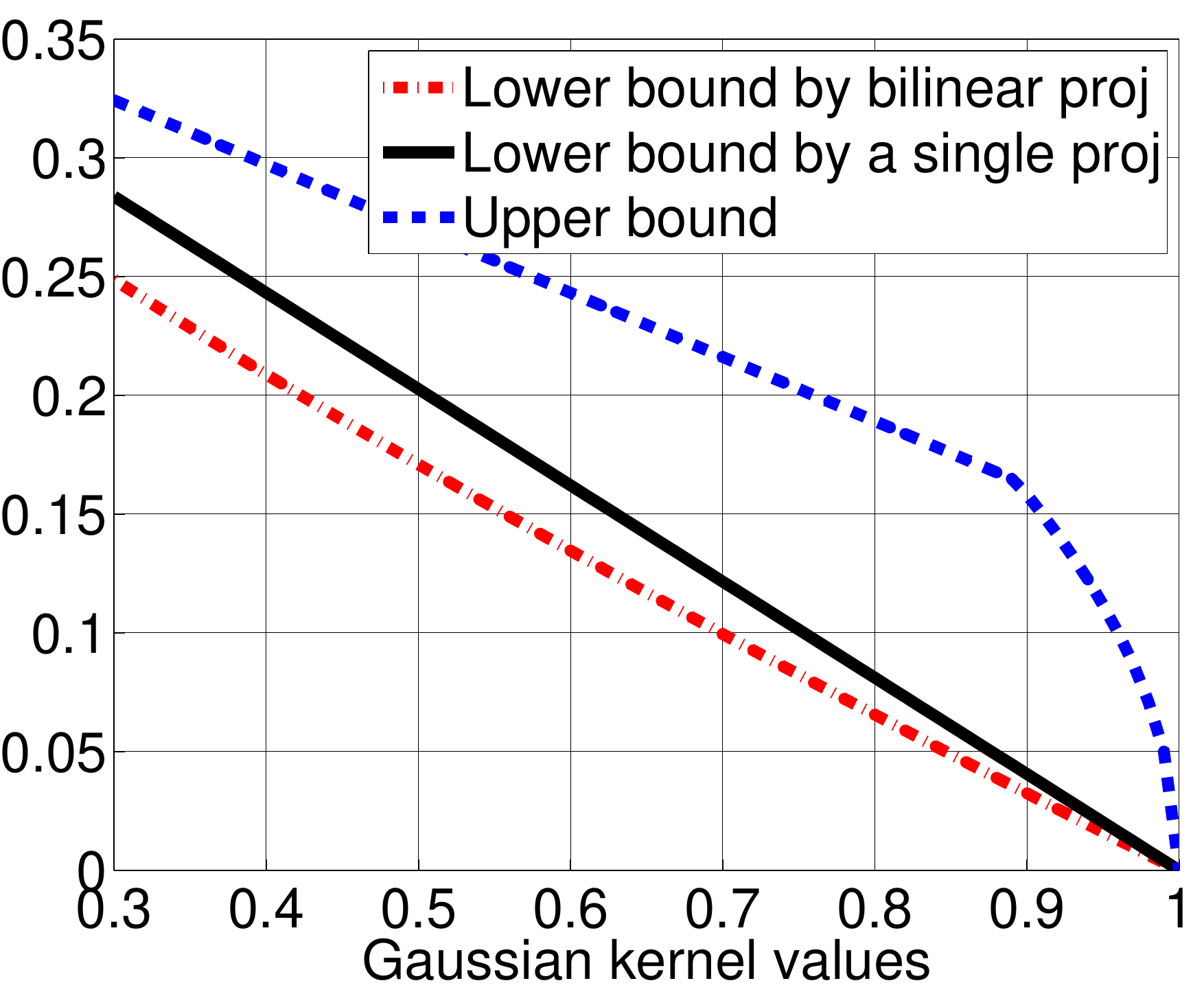}
\end{center}
\caption{Upper and lower bounds on the expected Hamming distance between binary codes computed by
BLSH-SIK and LSH-SIK.}
\label{fig:bounds}
\end{figure}

%:lemma 2
\begin{lemma}
\label{lem:ehd_blsh}
\bee
\lefteqn{  \E_{\bw, \bv, b, t} \Big[ \calI \big [ h(\bX) \neq h(\bY) \big] \Big] } \nonumber \\
& = &  \frac{8}{\pi^2} \sum_{m=1}^{\infty}\frac{1 - \kappa_{b} \big(m\bX-m\bY \big)}{4m^2-1}.
\eee
\normalsize
\end{lemma}

\begin{proof}
This is a slight modification of the result (for a randomized linear map) in \cite{RaginskyM2009nips}.
Since the proof is straightforward, it is placed in the supplementary material.
\end{proof}

Though the expected Hamming distance is analytically
represented with respect to $\kappa_{b}$,
its relationship with $\kappa_{g}$ is not fully exploited.
In order to figure out the similarity-preserving property of BLSH-SIK
in a more clear way, Theorem \ref{thm:bounds_blsh} is described
to show the upper and lower bounds on
the expected Hamming distance for BLSH-SIK in terms of $\kappa_g \big(\mbox{vec}(\bX-\bY) \big)$.

%:theorem 2
\begin{theorem}
\label{thm:bounds_blsh}
Define the functions
\bee
    g_1(\zeta) & \triangleq & \frac{4}{\pi^2} \left(1-\zeta^{0.79} \right), \\
    g_2(\zeta) & \triangleq & \min \left\{ \frac{1}{2}\sqrt{1-\zeta},
    \frac{4}{\pi^2} \left(1-\frac{2}{3} \zeta \right) \right\},
\eee
where $\zeta \in [0,1]$ and  $g_1(0)=g_2(0)=\frac{4}{\pi^2}$, $g_1(1)=g_2(1)=0$.
Gaussian kernel $\kappa_g$ is shift-invariant, normalized, and satisfies $\kappa_g(\alpha \bx - \alpha \by) \leq \kappa_g(\bx-\by)$
for any $\alpha \geq 1$.
Then the expected Hamming distance between any two embedded points computed by bilinear LSH-SIK satisfies
\be
    g_1 \Big(\kappa_g \big( \btau \big)  \Big) \leq  \E \Big[ \calI \big [ h(\bX) \neq h(\bY) \big] \Big]
    \leq g_2 \Big(\kappa_g \big(\btau \big) \Big),
\ee
where $\btau = \mbox{vec}(\bX-\bY)$.
\end{theorem}

\begin{proof}
We prove the upper and lower bound one at a time, following the technique used in \cite{RaginskyM2009nips}.
Note that the lower bound $g_1(\zeta)$ is slightly different from the one in Theorem \ref{thm:bounds_lsh},
however the upper bound $g_2(\zeta)$ is the same as the one in Theorem \ref{thm:bounds_lsh}.

\noindent
{\em Lower bound:} It follows from Property 2 and Lemma \ref{lem:rffbl} that
we can easily find the lower bound as
\small
\bee
    \E \Big[ \calI \big [ h(\bX) \neq h(\bY) \big] \Big]& \geq  &
    \frac{4}{\pi^2} \Big(1-\kappa_{b} \big(\mbox{vec}(\bX-\bY) \big) \Big)\\
    &\geq & \frac{4}{\pi^2} \Big(1-\kappa_{g}\big(\mbox{vec}(\bX-\bY)\big)^{0.79} \Big)\\
    & \triangleq & g_1 \Big(\kappa_g \big(\mbox{vec}(\bX-\bY) \big) \Big).
\eee
\normalsize

\noindent
{\em Upper bound:} By the proof of Lemma 2.3 \cite{RaginskyM2009nips},
we can easily find the upper bound as
\bee
\lefteqn{ \E \Big[ \calI \big [ h(\bX) \neq h(\bY) \big] \Big] } \\
 & \leq & \min\left\{ \frac{1}{2}\sqrt{1-\kappa_{b} \big( \btau \big)},
  \frac{4}{\pi^2} \left(1-\frac{2}{3}\kappa_{b} \big( \btau \big) \right) \right \}.
\eee
Moreover, the inequality $\kappa_g \big( \btau \big) \leq \kappa_{b} \big( \btau \big)$
in Lemma \ref{lem:rffbl} yields
\bee
\lefteqn{ \E \Big[ \calI \big [ h(\bX) \neq h(\bY) \big] \Big] } \\
& \leq & \min \left\{ \frac{1}{2}\sqrt{1-\kappa_g \big( \btau \big)},
\frac{4}{\pi^2} \left(1-\frac{2}{3}\kappa_g \big( \btau \big) \right) \right \}, \\
& \triangleq & g_2 \Big(\kappa_g \big( \btau \big) \Big).
\eee
\end{proof}

Theorem \ref{thm:bounds_blsh} shows that bilinear projections can generate similarity-preserving binary codes,
where the expected Hamming distance is upper and lower bounded in terms of $\kappa_g \big(\mbox{vec}(\bX-\bY) \big)$.
Compared with the original upper and lower bounds in case of a single projection shown in the Lemma 2.3 \cite{RaginskyM2009nips},
we derive the same upper bound and slightly worse lower bound as depicted in Fig. \ref{fig:bounds}.

\subsection{Bounds on Covariance}
\label{subsec:bounds_covariance}

\begin{figure}[ht]
\begin{center}
\includegraphics[scale=0.27]{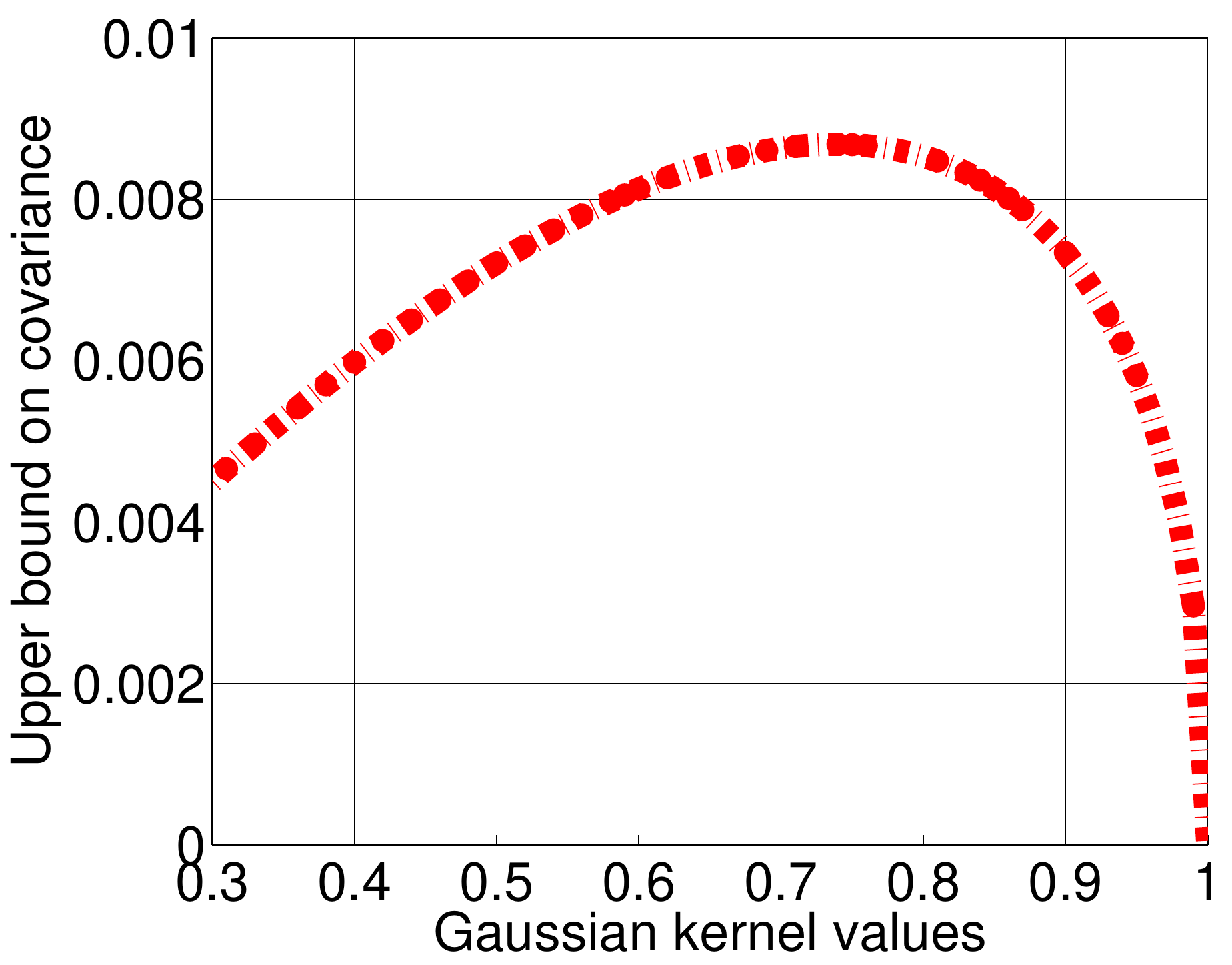}
\end{center}
\caption{Upper bound on covariance between the two bits induced by BLSH-SIK.
Horizontal axis suggests a Gaussian kernel value of two data points.
Vertical axis shows an upper bound on covariance.}
\label{fig:upper_cov}
\end{figure}

In this section, we analyze the covariance between two bits induced
by BLSH-SIK to address how much the performance would be
dropped compared with a single large projection matrix.

A hash function for multiple bits using bilinear projections (\ref{eq:Hbl})
implies that there exists the bits which share one of the projection vectors.
For example, assume that $h_i(\cdot)$ is given as
\be
    h_{i}(\bX) = \mbox{sgn} \Big(\cos(\bw_{1}^\top\bX\bv_{i}+b_i)+t_i \Big).
\ee
We can easily find the following $d_{v}-1$ hash functions
which shares $\bw_{1}$ with $h_{i}(\cdot)$.
\be
    h_{j}(\bX) = \mbox{sgn} \Big(\cos(\bw_{1}^\top\bX\bv_{j}+b_j)+t_j \Big),
\ee
where $j\in\{1,\cdots,d_{v}\}\setminus\{i\}$.

If the two bits does not share any one of projection vectors, the bits should be independent
which indicates a zero correlation. This phenomenon raises a natural question to ask that
how much the two bits, which share one of projection vectors, are correlated.
Intuitively, we expect that the highly correlated bits are not favorable,
because such bits does contain redundant information to approximate $ \E \Big[ \calI \big [ h(\bX) \neq h(\bY) \big] \Big] $.
Theorem \ref{thm:cov} shows that the upper bound on covariance between two bits
induced by bilinear projections is small, establishing the reason why
BLSH-SIK performs well enough in case of a large number of bits.

\begin{theorem}
\label{thm:cov}
Given the hash functions as Eq. (12-13), the upper bound on the covariance between
the two bits is derived as
\bee
    \mbox{cov} (\cdot) & \leq & \frac{64}{\pi^4} \left \{ \left( \sum_{m=1}^{\infty}
    \frac{\kappa_{g} \big(\mbox{vec}(\bX-\bY) \big)^{0.79m^2}}{4m^2-1} \right )^{2}  \right. \\
     & & \left. - \left(\sum_{m=1}^{\infty} \frac{\kappa_{g} \big(\mbox{vec}(\bX-\bY) \big)^{m^2}}{4m^2-1}
     \right)^2 \right\},
\eee
\normalsize
where $\kappa_g(\cdot)$ is the Gaussian kernel and
$\mbox{cov}(\cdot)$ is the covariance between two bits defined as
\small
\bee
    \mbox{cov}(\cdot) & =& \E \Big[ \calI \big [ h_i(\bX) \neq h_i(\bY) \big]  \calI \big [ h_j(\bX) \neq h_j(\bY) \big] \Big] \\
    & - & \E \Big[ \calI \big [ h_i(\bX) \neq h_i(\bY) \big] \Big] \E \Big[ \calI \big [ h_j(\bX) \neq h_j(\bY) \big] \Big].
\eee
\normalsize
\end{theorem}

\begin{proof}
Since the proof is lengthy and tedious, the detailed proof and
lower bound on the covariance can be found in the supplementary material.
\end{proof}

Fig. \ref{fig:upper_cov} depicts the upper bound on covariance between the two bits
induced by BLSH-SIK with respect to Gaussian kernel value.
We can easily see that the covariance between the two bits for the highly similar ($\kappa_g(\mbox{vec}(\bX-\bY))\approx 1$)
is nearly zero, indicating that there is no correlation between the two bits.
Unfortunately, there exists unfavorable correlation for the data points which is not highly (dis)similar. To remedy such unfavorable correlation, a simple heuristic is proposed,
in which $k\times m^2$ bits are first generated and randomly select the $k$ bits when
$k$ is the desired number of bits and $m$ is a free parameter for reducing the unfavorable correlation
trading-off storage and computational costs.
This simple heuristic reduces the correlation between
the two bits without incurring too much computational and storage costs.
Algorithm \ref{alg:Bilinear_LSH_SIK} summarizes the BLSH-SIK with
the proposed heuristic.

\begin{algorithm}[ht]
\caption{LSH for Shift-invariant Kernels with Bilinear Projections (BLSH-SIK)}
\label{alg:Bilinear_LSH_SIK}
\begin{algorithmic}[1]
\INPUT A data point is $\bX \in \Real^{d_w \times d_v}$, $k$ is the desired number of bits,
$m$ is the hyper-parameter to reduce the correlation, and $I$ is a subset with $k$ elements of $\{1,2,\cdots,m^2\times k\}$.
\OUTPUT A binary code of $\bX$ with $k$ bits.
\STATE $\bW \in \Real^{d_w \times m\sqrt{k}}$ and $\bV \in \Real^{d_v \times m\sqrt{k}}$ are element-wise drawn from the zero-mean Gaussian, $\mathcal{N}(0,1)$.
\STATE $\bb \in \Real^{mk}$ and $\bt \in \Real^{mk}$ are element-wise drawn from uniform distributions, $\mbox{Unif}[0, 2\pi]$ and $\mbox{Unif}[-1, +1]$, respectively.
\STATE Generate a binary code whose the number of bit is $k \times m^2$:
$\frac{1}{2}(1+\mbox{sgn}(\cos(\mbox{vec}(\bW^\top\bX\bV) + \bb) + \bt))$.
\STATE Select the $k$-bits from the binary code using the pre-defined subset $I$.
\end{algorithmic}
\end{algorithm}

\section{Experiments}
\label{sec:experiments}

\begin{figure*}[htp]
\begin{center}
\subfigure[400bits]{\includegraphics[scale=0.23]{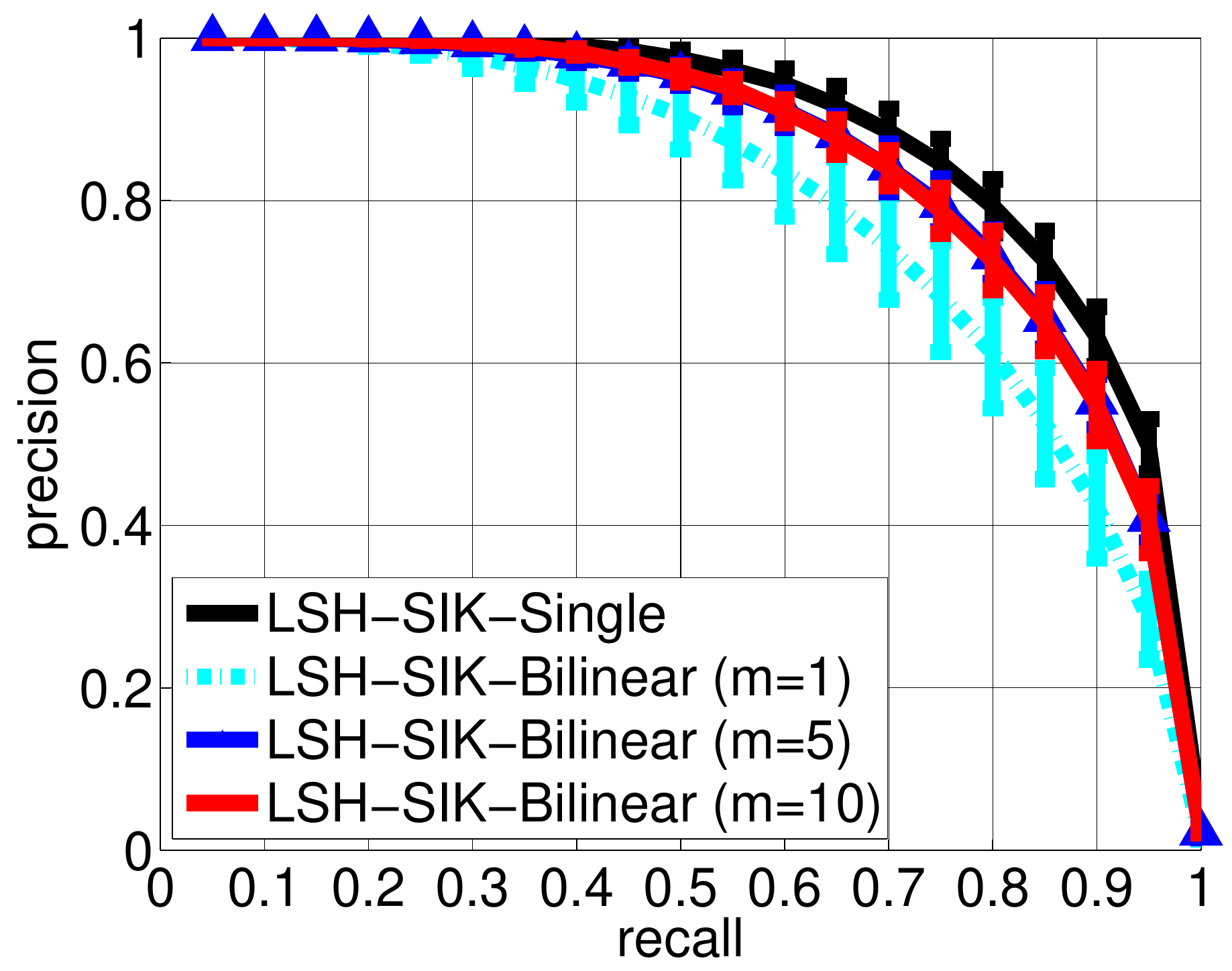}}
\subfigure[900bits]{\includegraphics[scale=0.23]{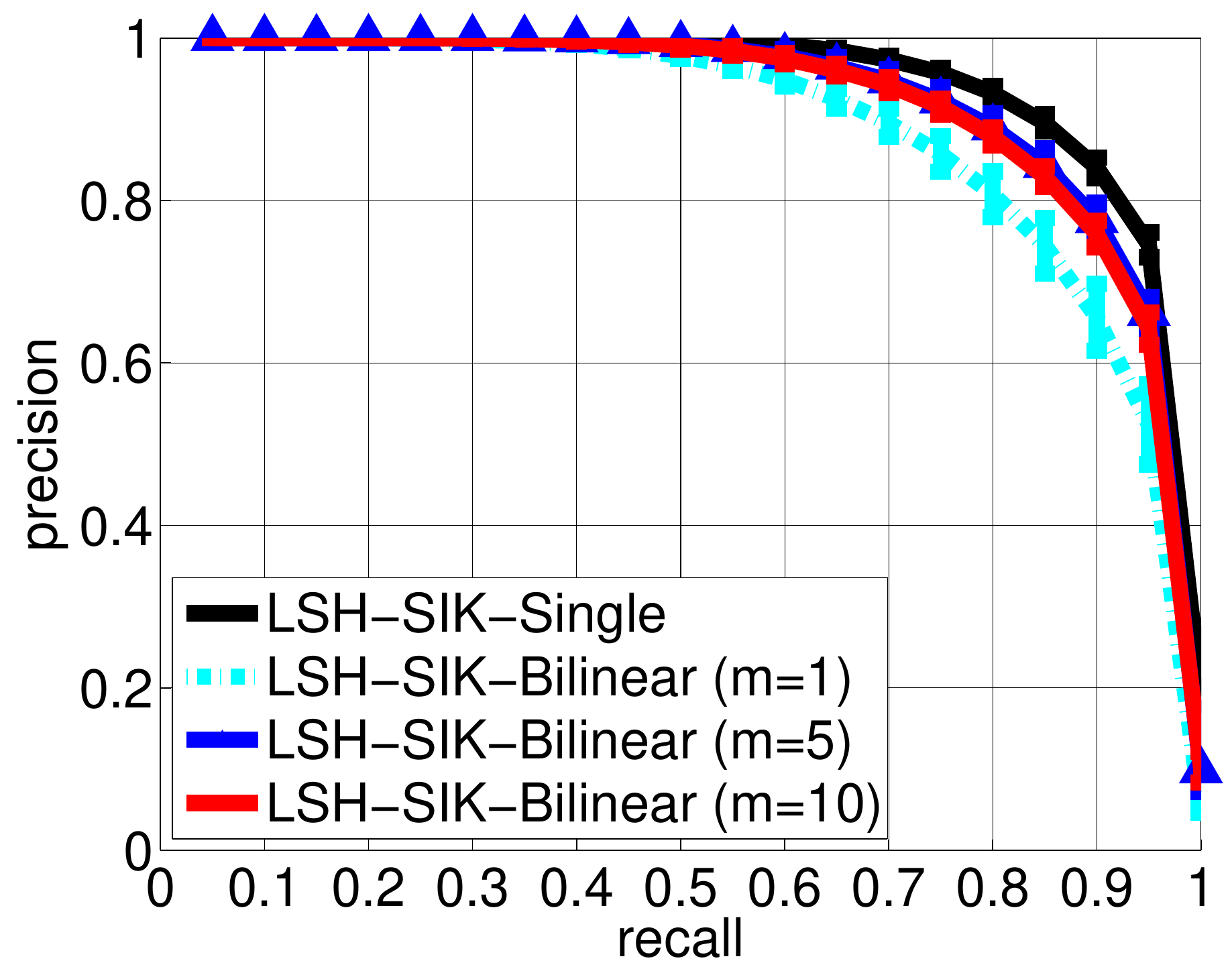}}
\subfigure[1,600bits]{\includegraphics[scale=0.23]{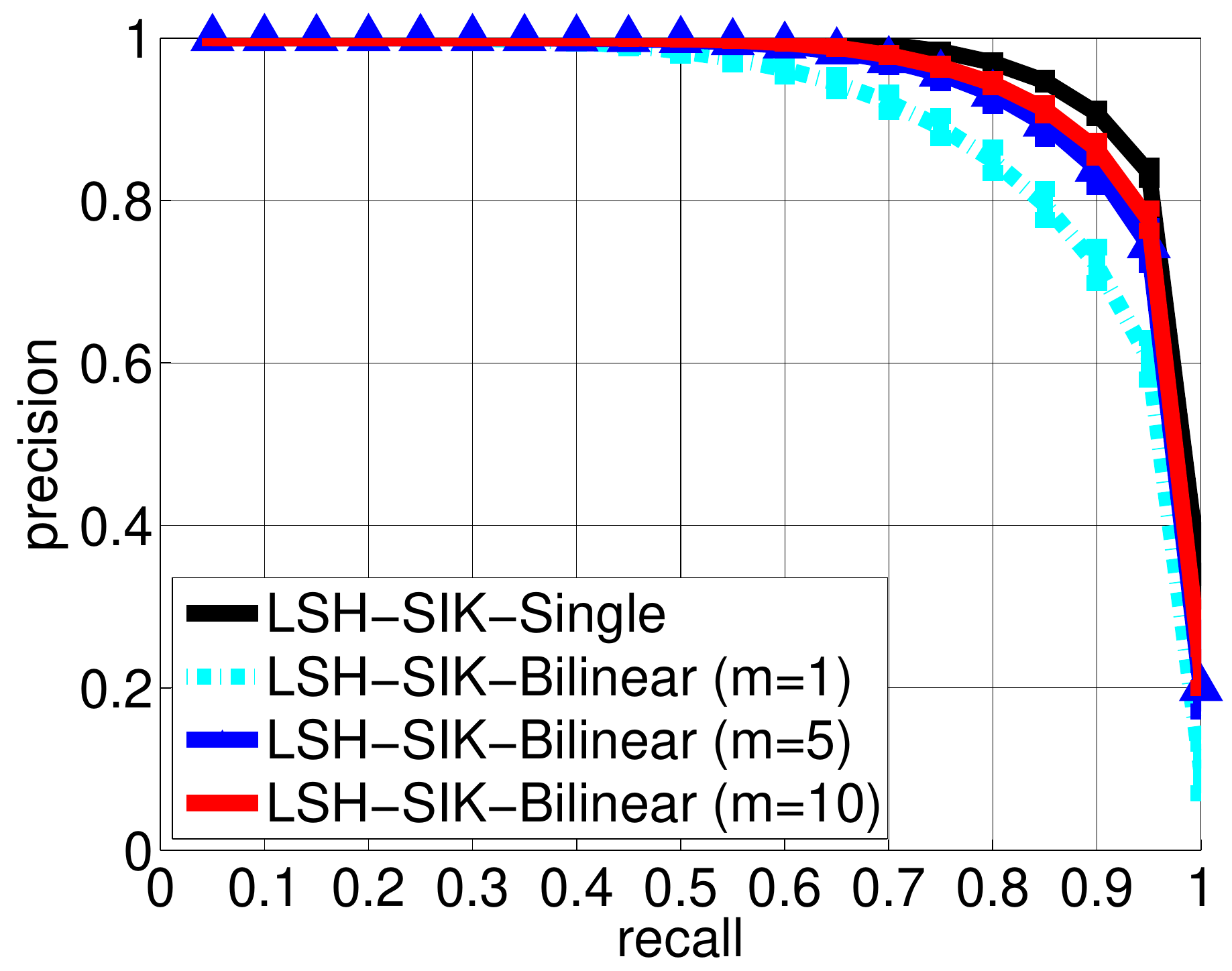}}
\subfigure[2,500bits]{\includegraphics[scale=0.23]{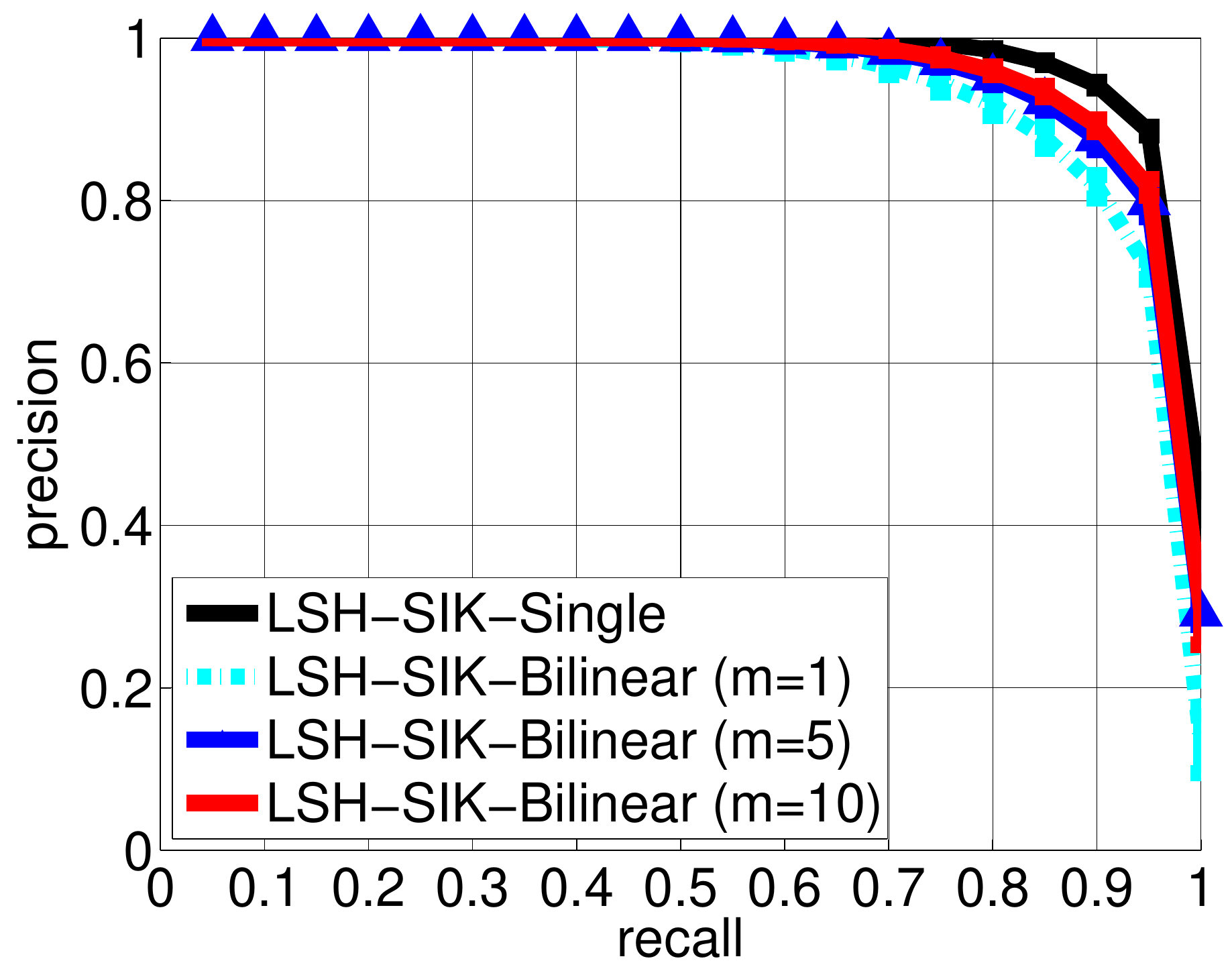}}
\end{center}
\caption{Precision-recall curves for LSH-SIK with a single projection (referred to as LSH-SIK-Single)
and BLSH-SIK (referred to as LSH-SIK-Bilinear) on MNIST with respect to the different number of bits.
In case of BLSH-SIK, the precision-recall curves are plotted for the different $m$,
which is introduced to reduce the correlation in Algorithm 1.}
\label{fig:mnist_comp}
\end{figure*}

In this section, we represent the numerical experimental results
to support the analysis presented in the previous sections,
validating the practical usefulness of BLSH-SIK.
For the numerical experiments, the two widely-used datasets,
MNIST \footnote{http://yann.lecun.com/exdb/mnist/} and Flickr45K \footnote{http://lear.inrialpes.fr/people/jegou/data.php}, are used to investigate the behaviors
of BLSH-SIK from small- to high-dimensional data.
MNIST consists of 70,000 handwritten digit images represented by a 28-by-28 matrix,
where the raw images are used for the experiments.
Flickr45K is constructed by randomly selecting 45,000 images from 1 million Flickr images
used in \cite{JegouH2010cvpr}.
VLAD \cite{JegouH2010cvpr} is used to represent an image with 500 cluster centers, resulting in a
$500 \times 128 = 64,000$ dimensional vector normalized to the unit length with $l_2$ norm.
For BLSH-SIK, we reshape an image into a 250-by-256 matrix.

The ground-truth neighbors should be carefully constructed for
comparing the hashing algorithm in a fair manner.
We adopt the same procedure to construct the ground-truth neighbors presented in \cite{RaginskyM2009nips}. First of all, we decide an appropriate threshold to judge which neighbors should be ground-truth neighbors, where the averaged Euclidean distance between the query and the 50th nearest neighbor is set to the appropriate threshold. Then, the ground-truth neighbor is decided if the distance between the query and the point is less than the threshold. Finally, we re-scale the dataset such that the threshold is one, leading that the scaling parameter for Gauassian kernel can be set to one. For both datasets, we randomly select 300 data points for queries, and the queries which has more than 5,000 ground-truth neighbors are excluded.
To avoid any biased results,
all precision-recall curves in this section are plotted by error bars with mean
and one standard deviation over 5 times repetition.

\begin{figure}[t]
\begin{center}
\subfigure[4,900bits]{\includegraphics[width=0.47\linewidth]
{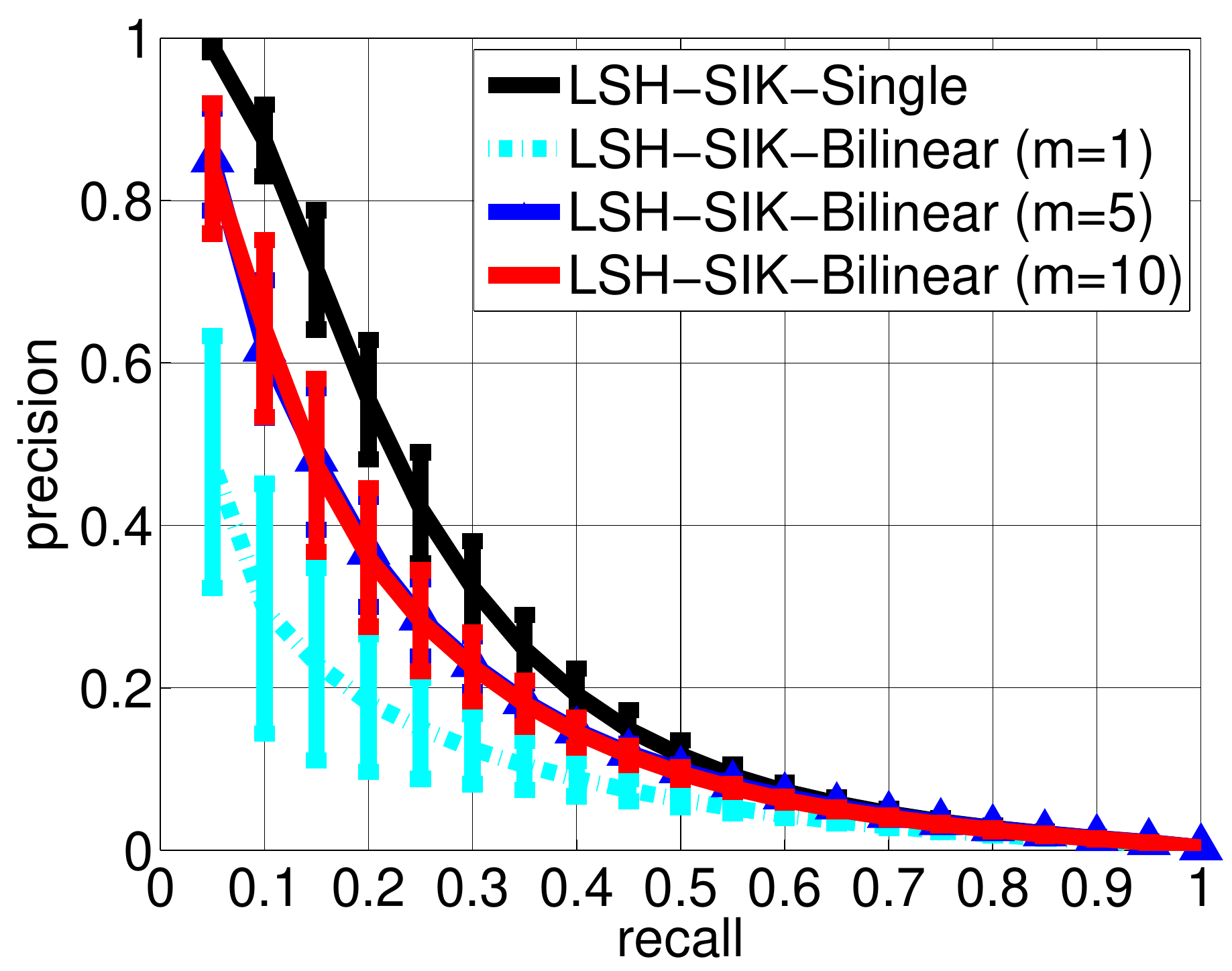}}
\subfigure[6,400bits]{\includegraphics[width=0.47\linewidth]
{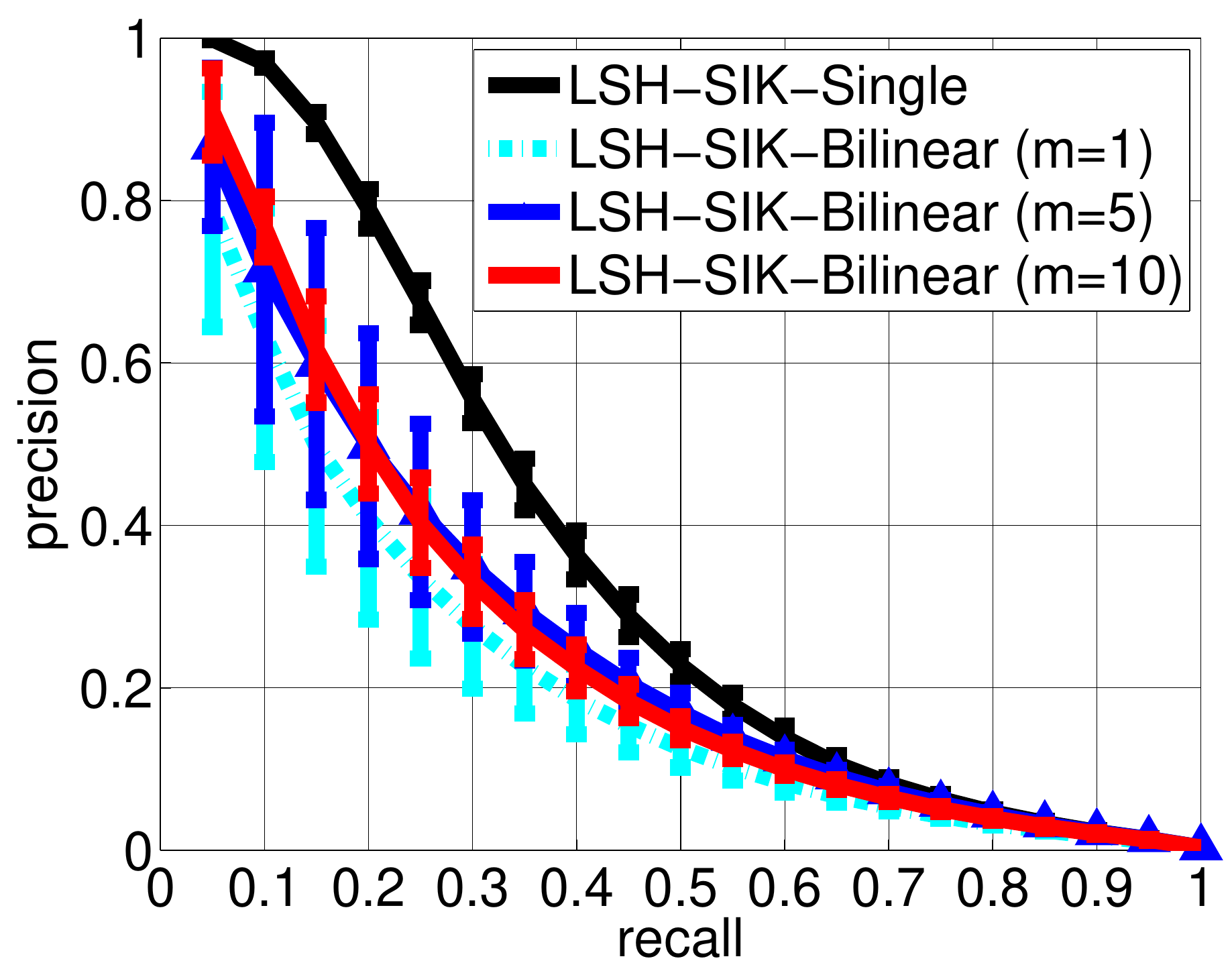}}
\end{center}
\caption{Precision-recall curves for LSH-SIK with a single projection (referred to as LSH-SIK-Single)
and BLSH-SIK (referred to as LSH-SIK-Bilinear)
on Flickr45K with respect to the different number of bits,
where the precision-recall curves for BLSH-SIK
are plotted for the different hyper-parameter $m$.}
\label{fig:flickr_comp}
\end{figure}

\begin{figure}[t]
\begin{center}
\subfigure[Computational time]{\includegraphics[width=0.47\linewidth]
{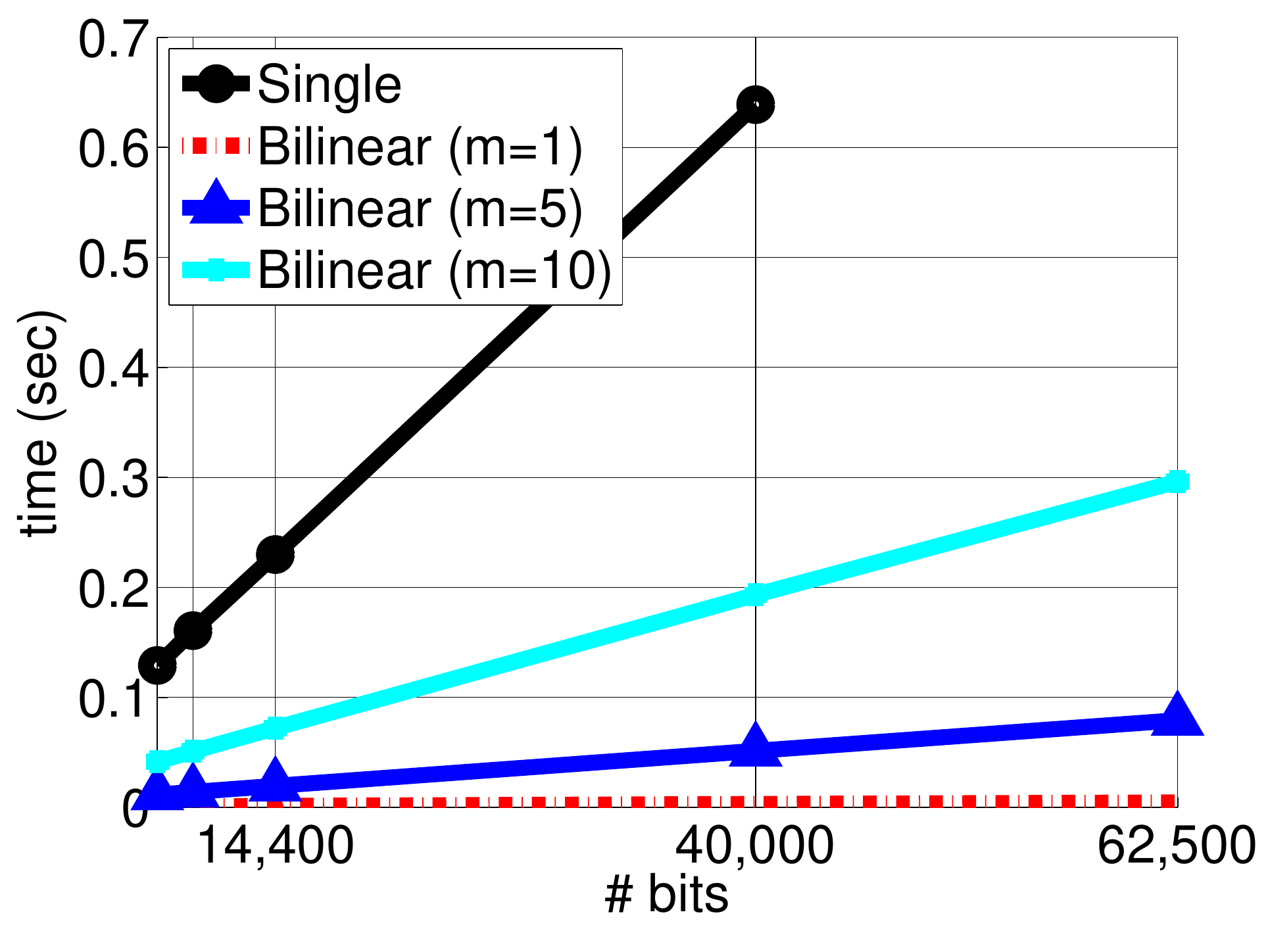}}
\subfigure[Memory consumption]{\includegraphics[width=0.47\linewidth]
{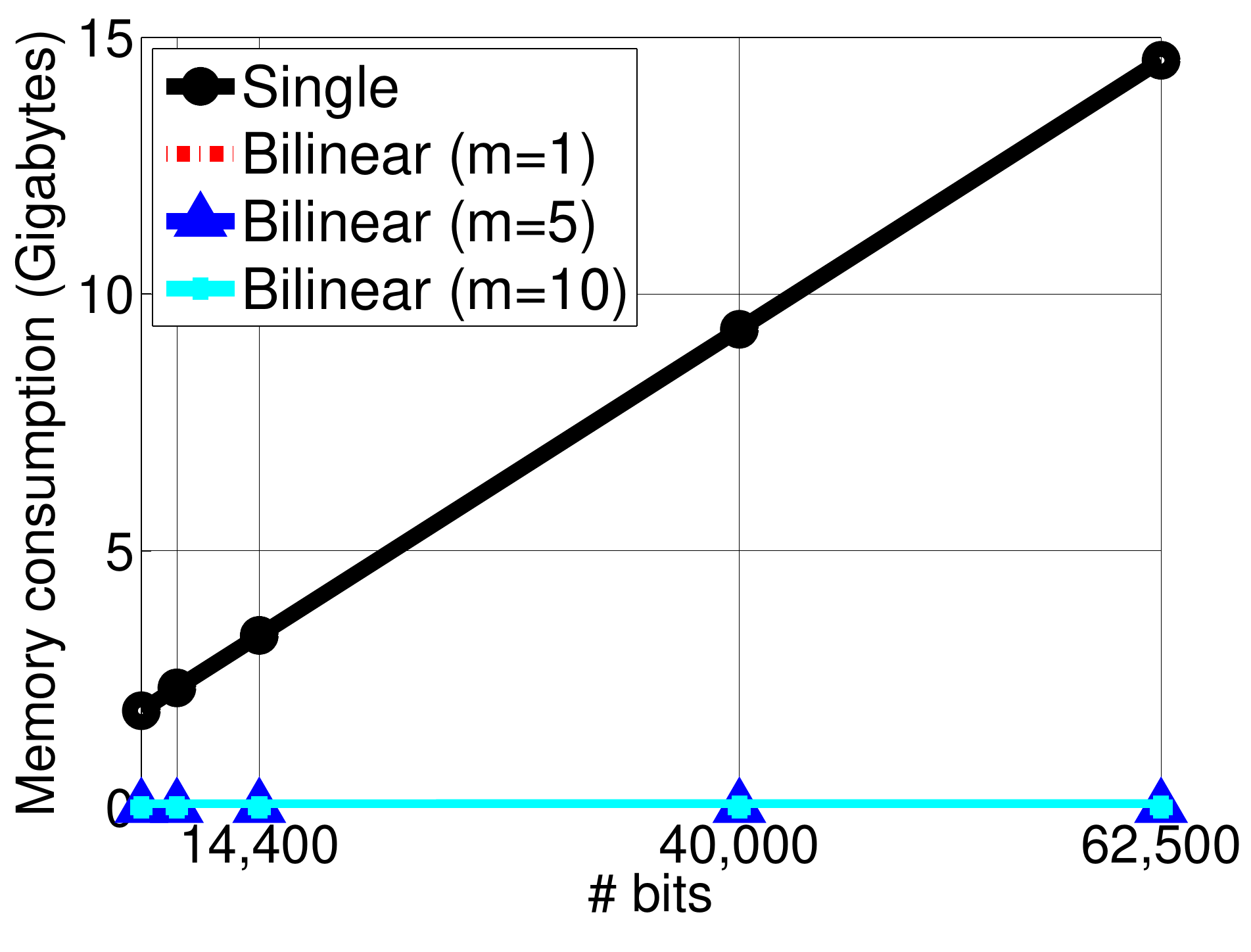}}
\end{center}
\caption{Comparison between LSH-SIK with a single large projection (referred to as Single)
and BLSH-SIK (referred to as Bilinear) in terms of the computational time
and memory consumption on the Flickr45K dataset. }
\label{fig:time_memory_comp}
\end{figure}

\begin{figure*}[htp]
\begin{center}
\subfigure[16,900 bits]{\includegraphics[scale=0.225]{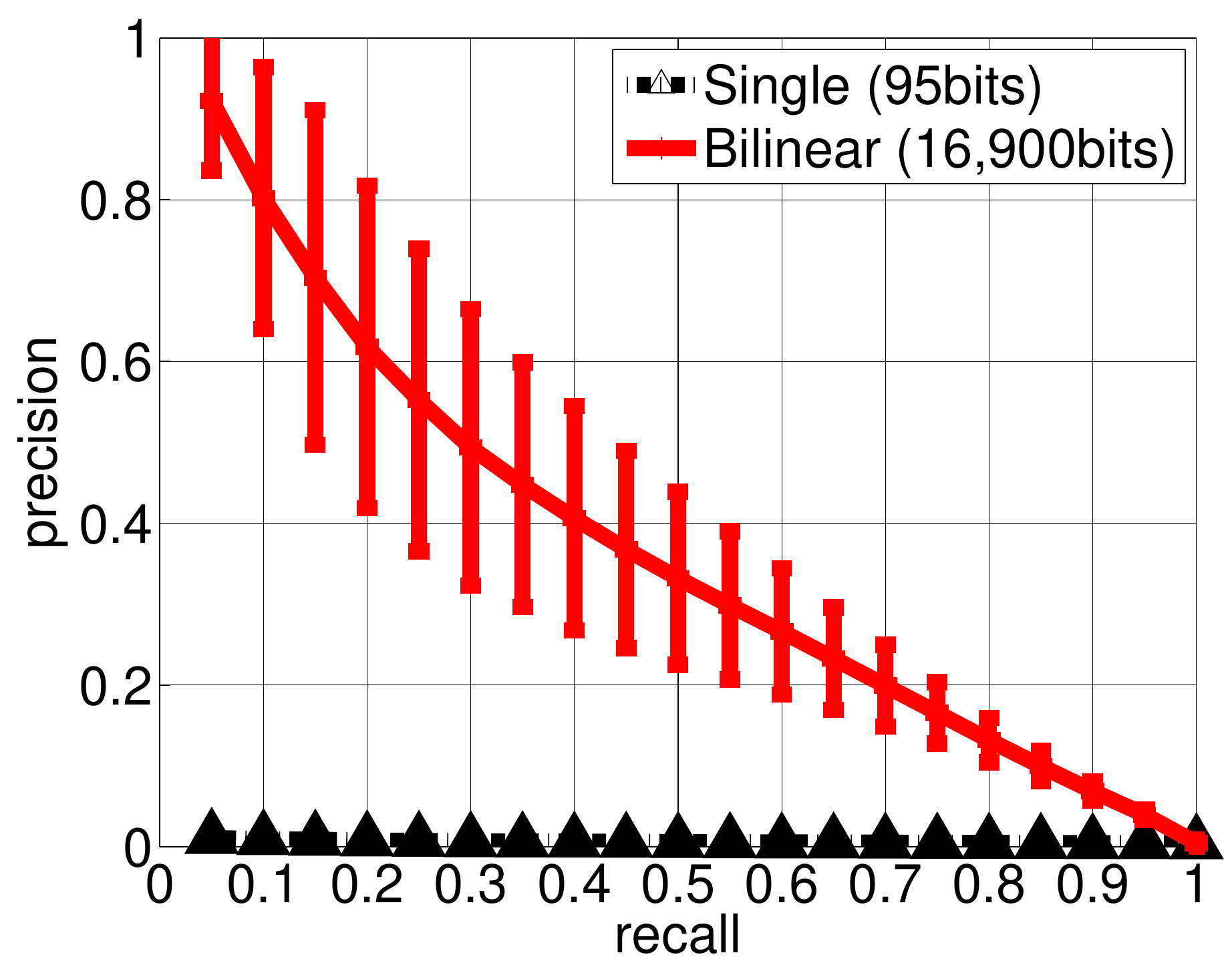}}
\subfigure[25,600 bits]{\includegraphics[scale=0.225]{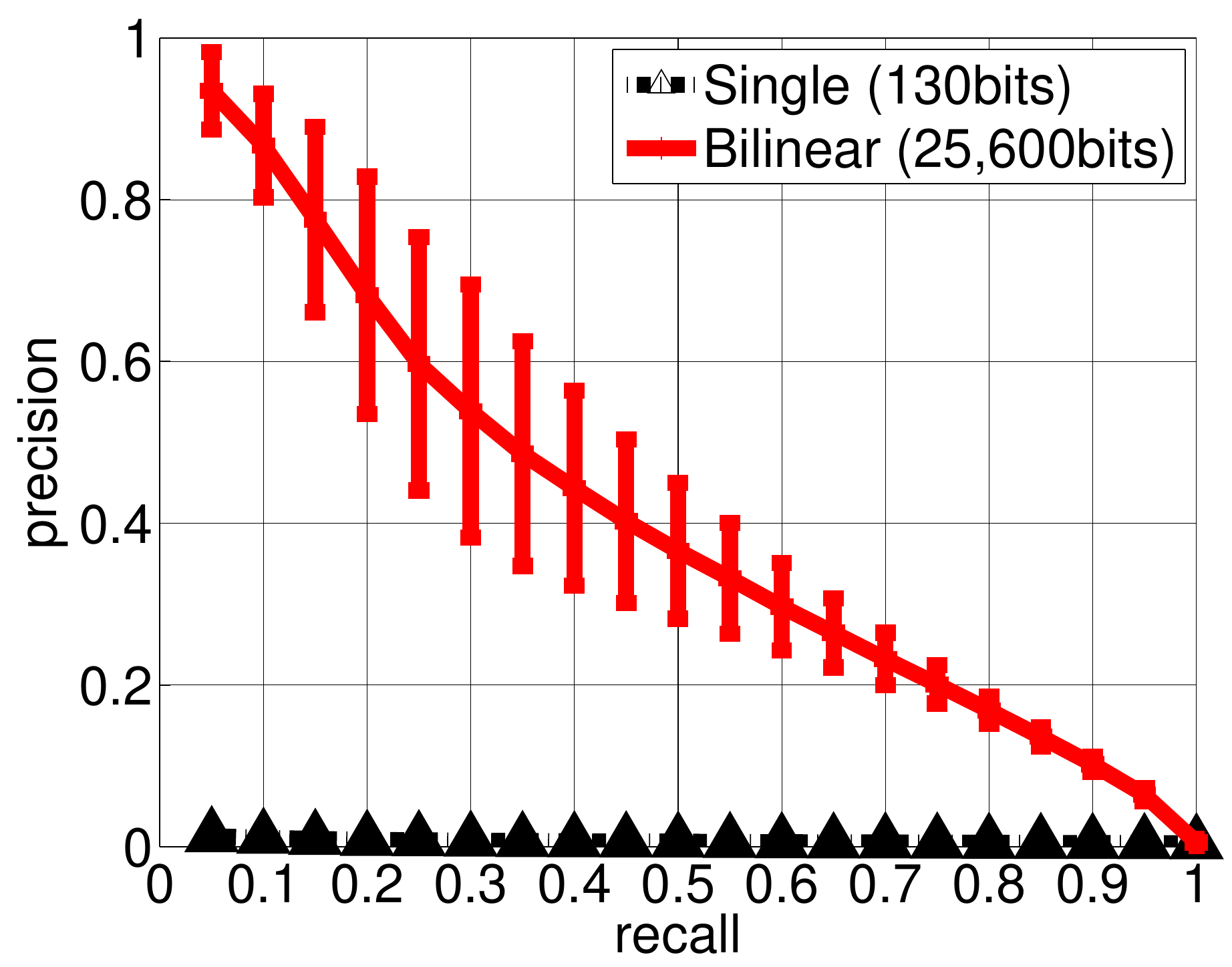}}
\subfigure[40,000 bits]{\includegraphics[scale=0.225]{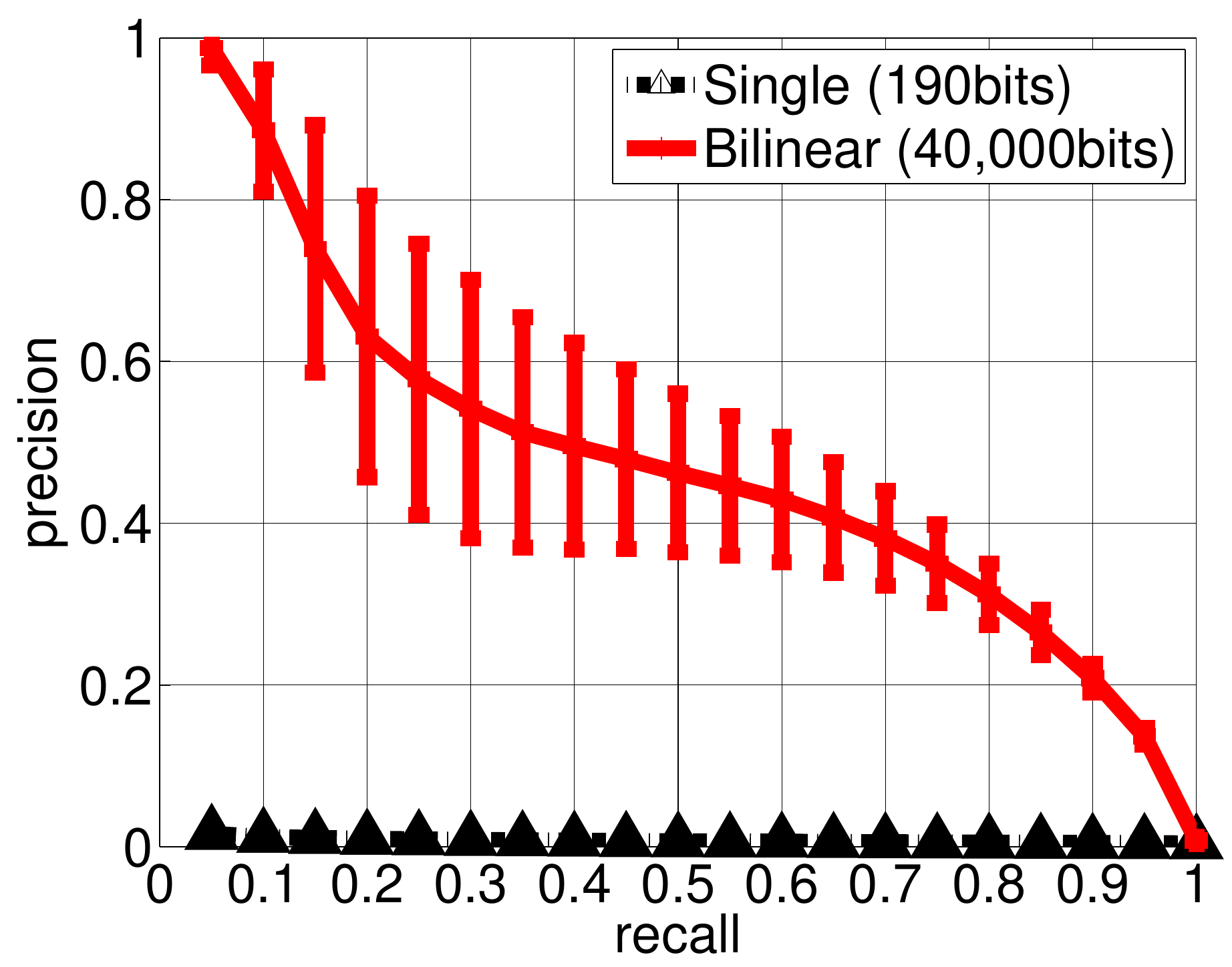}}
\subfigure[62,500 bits]{\includegraphics[scale=0.225]{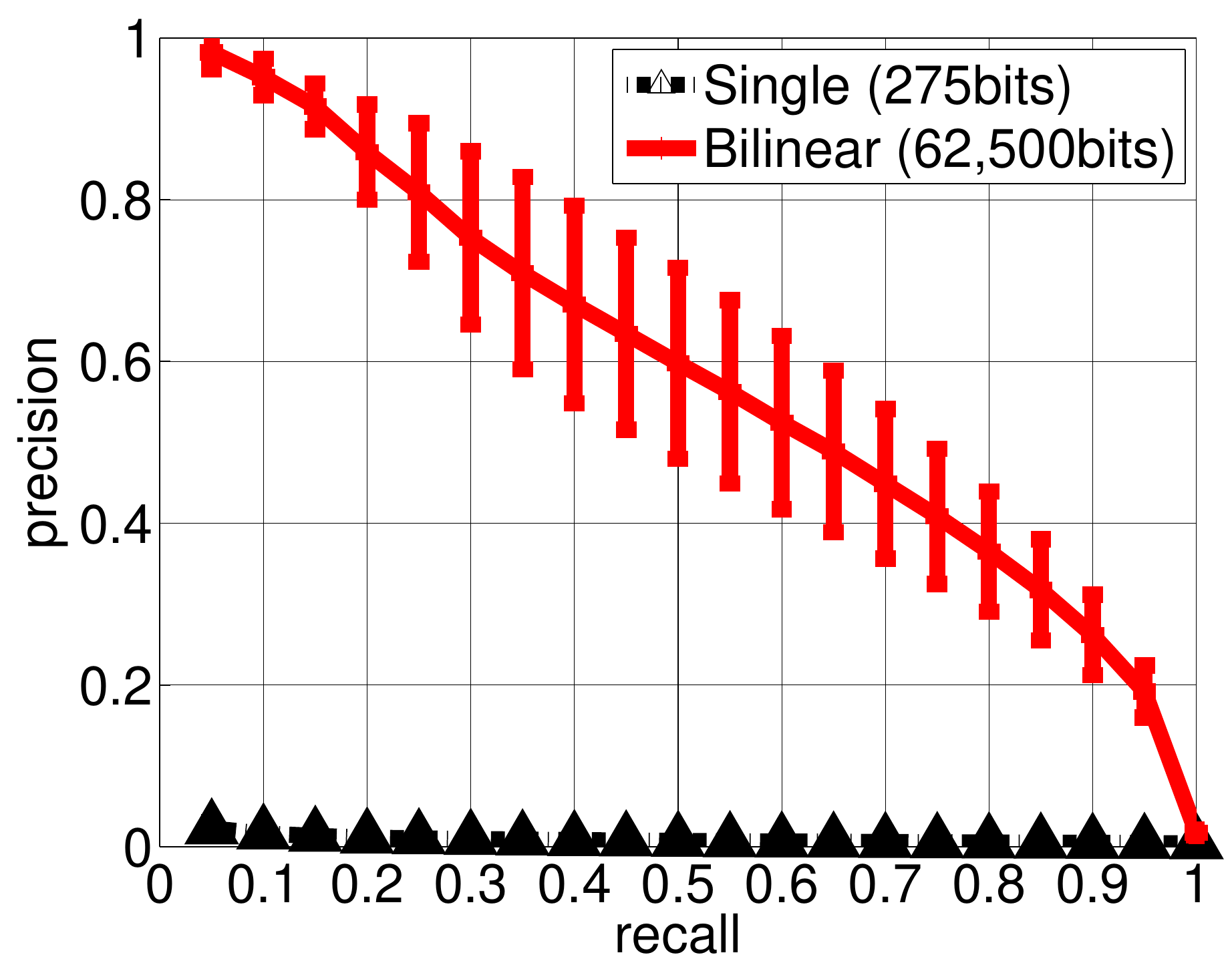}}

\subfigure[16,900 bits]{\includegraphics[scale=0.225]{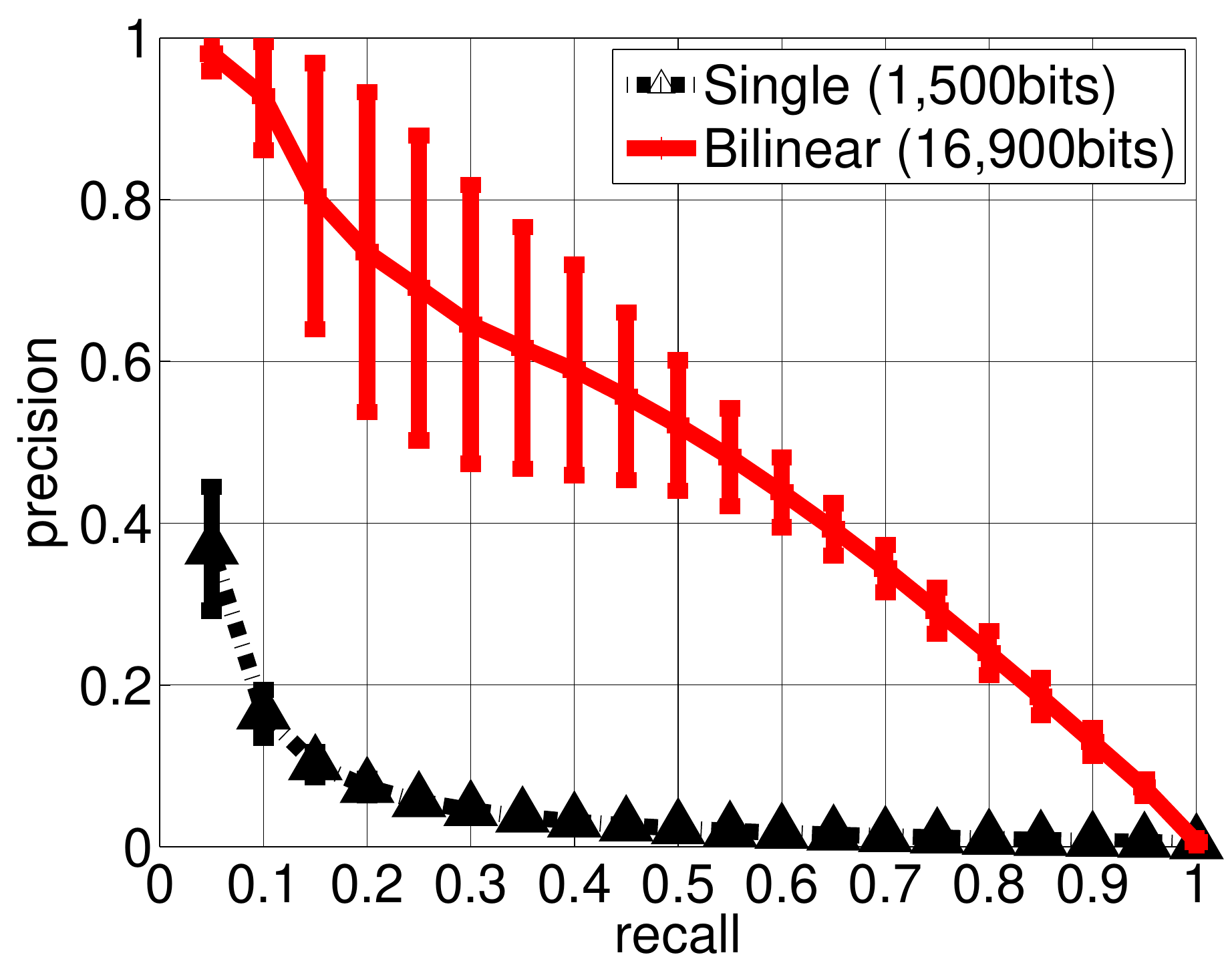}}
\subfigure[25,600 bits]{\includegraphics[scale=0.225]{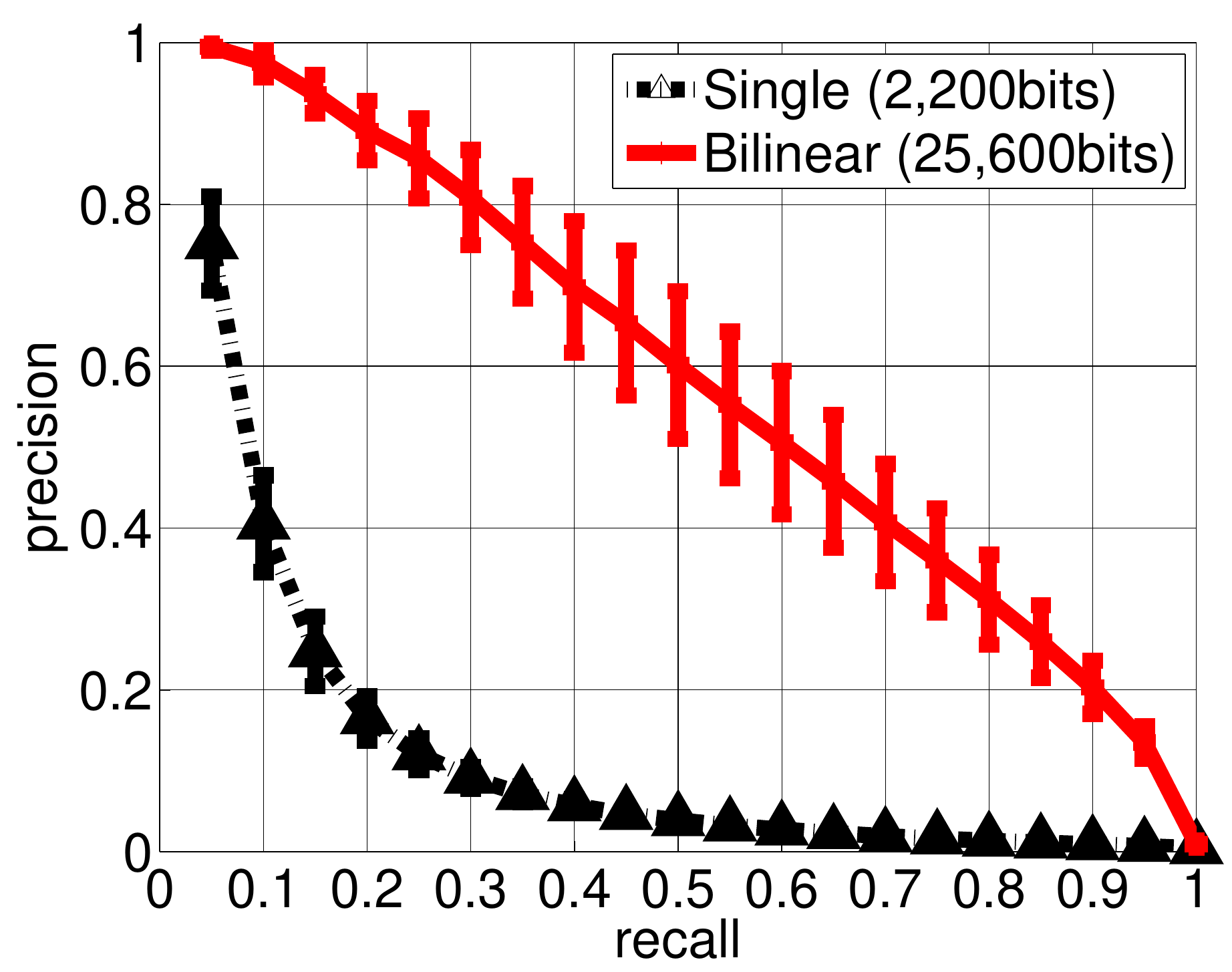}}
\subfigure[40,000 bits]{\includegraphics[scale=0.225]{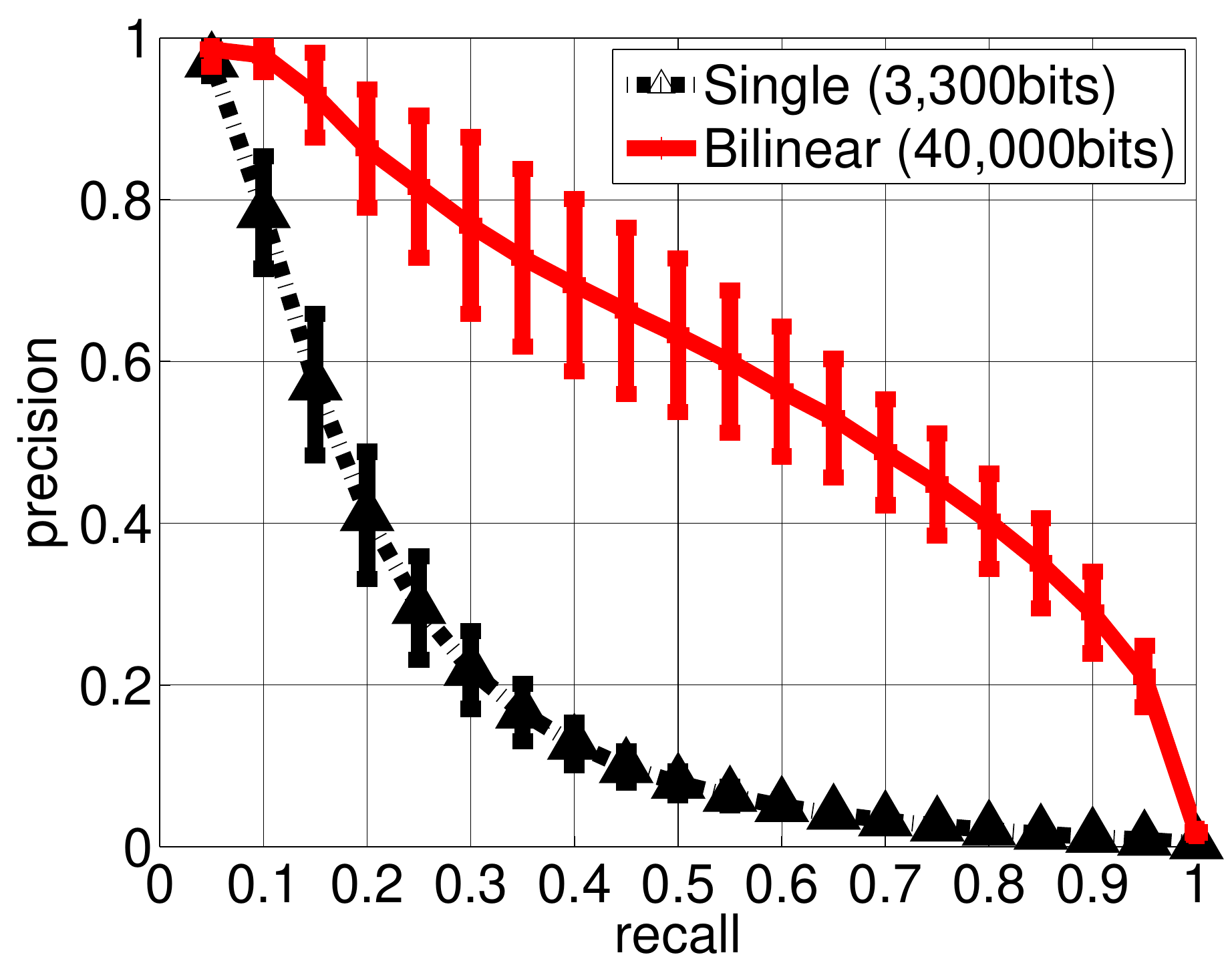}}
\subfigure[62,500 bits]{\includegraphics[scale=0.225]{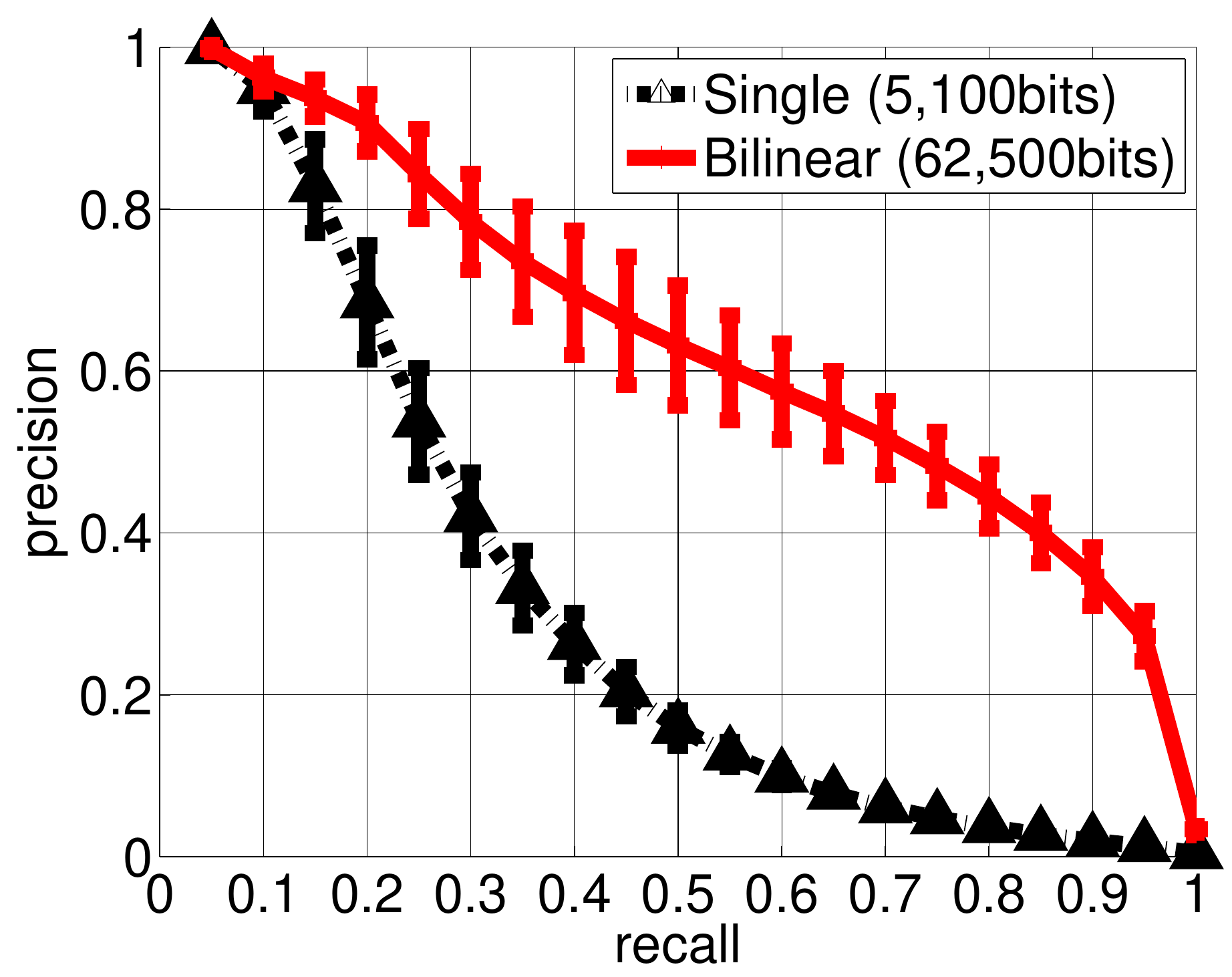}}
\end{center}
\caption{Precision-recall curves for LSH-SIK with a single projection (referred to as Single)
and BLSH-SIK (referred to as Bilinear) on Flickr45K when
the same computational time for generating a binary code is required to LSH-SIK with a single projection and BLSH-SIK. The first (second) row shows the results when $m$ of BLSH-SIK is one (five).}
\label{fig:time_limit}
\end{figure*}

Fig. \ref{fig:mnist_comp} and \ref{fig:flickr_comp} represent
precision-recall curves for LSH-SIK with a single projection
and BLSH-SIK on MNIST and Flickr45K with respect to the different number of bits.
In case of BLSH-SIK, the precision-recall curves are plotted for the different $m$,
which is introduced to reduce the correlation in Algorithm 1.
From the both figures, we observe that the larger $m$ helps to reduce the correlation of the bits
induced by BLSH-SIK. Even though BLSH-SIK cannot generate the same performance of LSH-SIK with a single projection, we argue that the performance is comparable.
Moreover, the computational time and memory consumption for
generating binary codes are significantly reduced as explained in the next paragraph.

Fig. \ref{fig:time_memory_comp} represents the comparison between
LSH-SIK with a single large projection and BLSH-SIK
in terms of the computational time \footnote{To measure the computational time, a single thread is used with a Intel i7 3.60GHz machine
(64GB main memory). Fig. \ref{fig:time_memory_comp} (a) does not include the computational time of LSH-SIK with a single projection for 62,500bits due to the high memory consumption.} and memory consumption on the Flickr45K dataset.
In case of BLSH-SIK, the time cost and memory consumption are reported with respect
to the different $m$, which evidently shows that
the computational time and memory consumption of BLSH-SIK are much smaller than
LSH-SIK with a single projection. From Fig. \ref{fig:mnist_comp}, \ref{fig:flickr_comp}
and \ref{fig:time_memory_comp}, we can conclude that $m=5$ is a good choice for BLSH-SIK,
because $m=5$ performs well compared to $m=10$ but it is much faster than $m=10$.

Fig. \ref{fig:time_limit} represents the
precision-recall curves for LSH-SIK with a single projection and BLSH-SIK on Flickr45K
with the same computational time limitation for generating a binary code.
Therefore, fewer bits are used for LSH-SIK with a single projection compared to BLSH-SIK.
For both $m=1$ and $m=5$, BLSH-SIK is superior to LSH-SIK with a single projection with the
same computational time.

\section{Conclusions}
\label{sec:conclusions}

In this paper we have presented a bilinear extension of LSH-SIK \cite{RaginskyM2009nips}, referred to as BLSH-SIK,
where we proved that the expected Hamming distance between the binary codes of two vectors is related to the value
of Gaussian kernel when column vectors of projection matrices are independently drawn from spherical Gaussian distribution.
Our theoretical analysis have confirmed that: (1) randomized bilinear projection yields similarity-preserving binary codes;
(2) the performance of BLSH-SIK is comparable to LSH-SIK, showing that the correlation between two bits of binary codes
computed by BLSH-SIK is small.
Numerical experiments on MNIST and Flickr45K datasets confirmed the validity of our method.
%As a future work, we are extending our analysis to other shift-invariant kernels, besides Gaussian kernel.

% /* acknowledgments */
\medskip
\noindent
{\bf Acknowledgements}:
This work was supported by National Research Foundation (NRF) of Korea (NRF-2013R1A2A2A01067464)
and  the IT R\&D Program of MSIP/IITP (B0101-15-0307, Machine Learning Center).

% /* references */
{\small
\bibliographystyle{ieee}
%\bibliography{/users/seungjin/pub/bib/sjc}
\bibliography{sjc}
}

% /* supplementary pages */
\newpage
\onecolumn
\appendix
\section{Bounds on the Expected Hamming Distance}

\begin{lemma}{\cite{RaginskyM2009nips}}
\label{lemma:t}
For any $u, v \in [-1, 1]$, $P_t\{\mbox{sgn}(u+t) \neq \mbox{sgn}(v+t)\} = |u-v|/2$.
\end{lemma}

\begin{lemma}{(Lemma 2 in the paper)}
\label{lemma:expected_hamming_dist}
\bee
    \E_{\bw, \bv, b, t}\big[\Ind_{h(\bX)\neq h(\bY)}\big] = 
    \frac{8}{\pi^2} \sum_{m=1}^{\infty}\frac{1 - \kappa_{bi}(m\bX-m\bY)}{4m^2-1},
\eee
\normalsize
where $h(\bX) \triangleq \frac{1}{2}(1 + \mbox{sgn}(\mbox{cos}(\bw^\top\bX\bv+b)+t))$,
$\bw, \bv \sim \calN(0,\bI)$, $b \sim \mbox{Unif}[0,2\pi]$, and $t \sim \mbox{Unif}[-1,1]$.
\end{lemma}
Proof. Using Lemma \ref{lemma:t}, we can show that $\E_{\bw, \bv, b, t}\big[\Ind_{h(\bX)\neq h(\bY)}\big] = \frac{1}{2}\E_{\bw,\bv,b} |\mbox{cos}(\bw^\top\bX\bv + b) - \mbox{cos}(\bw^\top\bY\bv + b)|$. By using a trigonometry identity,
\bee
    & & \frac{1}{2}\E_{b,\bw, \bv} |\mbox{cos}(\bw^\top\bX\bv + b) -
    \mbox{cos}(\bw^\top\bY\bv+b) |
     =  \frac{2}{\pi}\E_{\bw,\bv}
    \big| \mbox{sin}\big(\frac{\bw^\top(\bX-\bY)\bv}{2} \big) \big|.
\eee
By \cite{RaginskyM2009nips}, we use Fourier series of $g(\tau) = |\mbox{sin}(\tau)|$:
\bee
    g(\tau) = \frac{4}{\pi} \sum_{m=1}^{\infty} \frac{1-\mbox{cos}(2m\tau)}{4m^2-1}.
\eee
This formula leads to the following equation:
\bee
    \E_{\bw, \bv, b, t}\big[\Ind_{h(\bX)\neq h(\bY)}\big]
     = \frac{8}{\pi^2}  \sum_{m=1}^{\infty} \frac{1-\E_{\bw, \bv}\big [\mbox{cos}(m\bw^\top(\bX-\bY)\bv)\big ]}{4m^2-1}
\eee
According to the proof of Lemma 1 described in the paper,
we know that $\E_{\bw,\bv}\mbox{cos}(m\bw^\top(\bX-\bY)\bv) = \kappa_{bi}(m\bX-m\bY)$,
which completes the proof.
\qquad  Q.E.D.
\newline

\section{Bounds on the Covariance by Bilinear Projections}

\begin{lemma}
Given a datum as $\bX \in \mathcal{R}^{d_1 \times d_2}$ and bilinear projections, $\bw, \bv_1, \bv_2$,
are drawn from the $\mathcal{N}(\mathbf{0}, \bI)$,
$\E_{\bw, \bv_1, \bv_2}\big [\mbox{cos}(m\bw^\top \bX \bv_1)\mbox{cos}(n\bw^\top \bX \bv_2) \big]
= \kappa_{bi}(\sqrt{(m^2+n^2)}\bX)$.
\end{lemma}
Proof.
\bee
    & & \E_{\bw, \bv_1, \bv_2}\big [\mbox{cos}(m\bw^\top \bX \bv_1)\mbox{cos}(n\bw^\top \bX \bv_2) \big] \\
    & = & \int\big[ \int \mbox{cos}(m\bw^\top\bX\bv_1)p(\bv_1)d\bv_1\big]
            \big[ \int \mbox{cos}(m\bw^\top\bX\bv_2)p(\bv_2)d\bv_2\big] p(\bw)d\bw \\
    & = & \int \kappa_g(m\bw^\top\bX)\kappa_g(n\bw^\top\bX) p(\bw)d\bw \quad \mbox{(by Lemma 1 in the paper)} \\
    & = & \int \frac{1}{\sqrt{2\pi}} \mbox{exp}(-\frac{1}{2}(\bw^\top [\bI + m^2\bX\bX^\top + n^2\bX\bX^\top]\bw))d\bw \\
    & = & |\bI + (m^2+n^2)\bX\bX^\top |^{-\frac{1}{2}} \triangleq \kappa_{bi}(\sqrt{(m^2+n^2)}\bX).
\eee
%\newline

\begin{theorem}{(Theorem 3 in the paper)}
\label{theorem:cov}
Given the hash functions $h_1(\cdot)$ and $h_2(\cdot)$, the upper bound on the covariance between
the two bits is derived as
%\small
\bee
    \mbox{cov}(\cdot) \leq (\frac{64}{\pi^4}) \Big [ \big ( \sum_{m=1}^{\infty} \frac{\kappa_{g}(\mbox{vec}(\bX-\bY))^{0.79m^2}}{4m^2-1} \big )^{2}
      - \big (\sum_{m=1}^{\infty} \frac{\kappa_{g}(\mbox{vec}(\bX-\bY))^{m^2}}{4m^2-1}
     \big)^2 \Big ],
\eee
\normalsize
where $\kappa_g(\cdot)$ is the Gaussian kernel and
$\mbox{cov}(\cdot)$ is the covariance between two bits defined as
%\small
\bee
    \mbox{cov}(\cdot) & = & \E_{\bw, \bv_1, \bv_2, b_1, b_2, t_1, t_2} \big[\Ind_{h_{1}(\bX)\neq h_{1}(\bY)}    \Ind_{h_{2}(\bX)\neq h_{2}(\bY)} \big] \\
     & - & \E_{\bw, \bv_1, b_1, t_1}\big[\Ind_{h_{1}(\bX)\neq h_{1}(\bY)}\big]
       \E_{\bw, \bv_2, b_2, t_2}\big[\Ind_{h_{2}(\bX)\neq h_{2}(\bY)}\big], \\
    h_{1}(\bX) & = & \mbox{sgn} \Big(\cos(\bw^\top\bX\bv_{1}+b_1)+t_1 \Big),\\
    h_{2}(\bX) & = & \mbox{sgn} \Big(\cos(\bw^\top\bX\bv_{2}+b_2)+t_2 \Big).
\eee
\normalsize
\end{theorem}
Proof. First, we want to derive the first term in the covariance in terms of $\kappa_{bi}(\cdot)$.
\small
\bee
    & & \E_{\bw, \bv_1, \bv_2, b_1, b_2, t_1, t_2}\big[\Ind_{h_{1}(\bX)\neq h_{1}(\bY)}
    \Ind_{h_{2}(\bX)\neq h_{2}(\bY)} \big] \\
    & = & \frac{1}{4}\E_{\bw, \bv_1, \bv_2, b_1, b_2}\big[|\mbox{cos}(\bw^\top\bX\bv_1 + b_1) - \mbox{cos}(\bw^\top\bY\bv_1+b_1)|
    |\mbox{cos}(\bw^\top\bX\bv_2 + b_2) - \mbox{cos}(\bw^\top\bY\bv_2+b_2)| \big] \quad \mbox{($\because$ Lemma \ref{lemma:t})}\\
    & = & \frac{4}{\pi^2} \E_{\bw, \bv_1, \bv_2} \big[ |\mbox{sin}(\frac{\bw^\top(\bX-\bY)\bv_1}{2}) \mbox{sin}(\frac{\bw^\top(\bX-\bY)\bv_2}{2}) | \big] \\
    & = & (\frac{64}{\pi^4}) \sum_{m,n=1}^{\infty}
    \E_{\bw, \bv_1, \bv_2} \big[ \big( \frac{1-\mbox{cos}(m\bw^\top(\bX-\bY)\bv_1)}{4m^2-1} \big)     \big( \frac{1-\mbox{cos}(n\bw^\top(\bX-\bY)\bv_2)}{4n^2-1} \big) \big ]
\eee
\normalsize
Using the Lemma 2, the first term in the covariance can be represented in terms of $\kappa_{bi}(\cdot)$:
%\small
\bee
    & & \E_{\bw, \bv_1, b_1, b_2, t_1, t_2}\big[\Ind_{h_{1}(\bX)\neq h_{1}(\bY)}
    \Ind_{h_{2}(\bX)\neq h_{2}(\bY)} \big] \\
    & = & (\frac{64}{\pi^4}) \sum_{m,n=1}^{\infty} \frac{1}{4m^2-1}\frac{1}{4n^2-1}
    \big (1 - \kappa_{bi}(m\bX-m\bY) - \kappa_{bi}(n\bX-n\bY) + \kappa_{bi}(\sqrt{(m^2+n^2)}(\bX-\bY)) \big )
\eee
\normalsize
The second term in the covariance is also represented in terms of $\kappa_{bi}(\cdot)$:
%\small
\bee
    & & \E_{\bw, \bv_1, b_1, t_1}\big[\Ind_{h_{1}(\bX)\neq h_{1}(\bY)}\big]
       \E_{\bw, \bv_2, b_2, t_2}\big[\Ind_{h_{2}(\bX)\neq h_{2}(\bY)}\big] \\
    & = & (\frac{64}{\pi^4}) \sum_{m,n=1}^{\infty} \frac{1}{4m^2-1}\frac{1}{4n^2-1}
    \big (1 - \kappa_{bi}(m(\bX-\bY)) - \kappa_{bi}(n(\bX-\bY))
    + \kappa_{bi}(m(\bX-\bY))\kappa_{bi}(n(\bX-\bY)) \big )
\eee
\normalsize
Therefore, the covariance between two bits is computed as
\small
\bee
    \mbox{cov}(\cdot)
    & = & (\frac{64}{\pi^4}) \sum_{m,n=1}^{\infty} \frac{1}{4m^2-1}\frac{1}{4n^2-1}
    \big [ \kappa_{bi}(\sqrt{(m^2+n^2)}(\bX-\bY)) - \kappa_{bi}(m(\bX-\bY))\kappa_{bi}(n(\bX-\bY))
    \big ] \\
    & \leq & (\frac{64}{\pi^4}) \sum_{m,n=1}^{\infty} \frac{1}{4m^2-1}\frac{1}{4n^2-1}
    \big [ \kappa_{g}(\mbox{vec}(\sqrt{(m^2+n^2)}(\bX-\bY)))^{0.79}
    - \kappa_{g}(\mbox{vec}(m\bX-m\bY))\kappa_{g}(\mbox{vec}(n\bX-n\bY))
    \big ] \\
    & = & (\frac{64}{\pi^4}) \sum_{m,n=1}^{\infty} \frac{1}{4m^2-1}\frac{1}{4n^2-1}
    \big [ \kappa_{g}(\mbox{vec}(\bX-\bY))^{0.79(m^2+n^2)}
    - \kappa_{g}(\mbox{vec}(\bX-\bY))^{m^2}\kappa_{g}(\mbox{vec}(\bX-\bY))^{n^2}
    \big ] \\
    & = & (\frac{64}{\pi^4}) \big [ \big ( \sum_{m=1}^{\infty} \frac{\kappa_{g}(\mbox{vec}(\bX-\bY))^{0.79m^2}}{4m^2-1} \big )^{2}
     - \big (\sum_{m=1}^{\infty} \frac{\kappa_{g}(\mbox{vec}(\bX-\bY))^{m^2}}{4m^2-1}
     \big)^2 \big ],
\eee
\normalsize
where the second inequality is given by Lemma 1 in the paper ($\kappa_g(\mbox{vec}(\bX-\bY)) \leq \kappa_{bi}(\bX-\bY) \leq \kappa_g(\mbox{vec}(\bX-\bY))^{0.79}$) and
the third equality is given by $\kappa_g(\mbox{vec}(m\bX-m\bY)) = \kappa_g(\mbox{vec}(\bX-\bY))^{m^2}$.
The lower bound can be derived in a similar way.

\begin{corollary}
Given the hash functions $h_1(\cdot)$ and $h_2(\cdot)$, the lower bound on the covariance between
the two bits is derived as
%\small
\bee
    \mbox{cov}(\cdot) \geq (\frac{64}{\pi^4}) \Big [ \big ( \sum_{m=1}^{\infty} \frac{\kappa_{g}(\mbox{vec}(\bX-\bY))^{m^2}}{4m^2-1} \big )^{2}
      - \big (\sum_{m=1}^{\infty} \frac{\kappa_{g}(\mbox{vec}(\bX-\bY))^{0.79m^2}}{4m^2-1}
     \big)^2 \Big ],
\eee
\normalsize
where $\kappa_g(\cdot)$ is the Gaussian kernel and
$\mbox{cov}(\cdot)$ is the covariance between two bits defined as
%\small
\bee
    \mbox{cov}(\cdot) & = & \E_{\bw, \bv_1, \bv_2, b_1, b_2, t_1, t_2} \big[\Ind_{h_{1}(\bX)\neq h_{1}(\bY)}    \Ind_{h_{2}(\bX)\neq h_{2}(\bY)} \big] \\
     & - & \E_{\bw, \bv_1, b_1, t_1}\big[\Ind_{h_{1}(\bX)\neq h_{1}(\bY)}\big]
       \E_{\bw, \bv_2, b_2, t_2}\big[\Ind_{h_{2}(\bX)\neq h_{2}(\bY)}\big], \\
    h_{1}(\bX) & = & \mbox{sgn} \Big(\cos(\bw^\top\bX\bv_{1}+b_1)+t_1 \Big),\\
    h_{2}(\bX) & = & \mbox{sgn} \Big(\cos(\bw^\top\bX\bv_{2}+b_2)+t_2 \Big).
\eee
\normalsize
\end{corollary}

\end{document}